\documentclass[journal]{IEEEtran}

\usepackage{moreverb,url}
\usepackage[colorlinks,bookmarksopen,bookmarksnumbered,citecolor=red,urlcolor=red]{hyperref}

\newcommand\BibTeX{{\rmfamily B\kern-.05em \textsc{i\kern-.025em b}\kern-.08em
T\kern-.1667em\lower.7ex\hbox{E}\kern-.125emX}}

\def\volumeyear{2021}

\usepackage{amssymb}
\usepackage{amsthm}
\usepackage{tabularx,colortbl}
\usepackage{array}
\usepackage{algorithmic}
\usepackage{algorithm}
\usepackage{graphicx}
\usepackage{epsfig}
\usepackage{hyperref}

\newcolumntype{L}[1]{>{\raggedright\arraybackslash}m{#1}}
\newcolumntype{C}[1]{>{\centering\arraybackslash}m{#1}}
\newcolumntype{R}[1]{>{\raggedleft\arraybackslash}m{#1}}
\newcommand\norm[1]{\left\lVert#1\right\rVert}

\hypersetup{
    colorlinks=true,
    linkcolor=blue,
    filecolor=magenta,      
    urlcolor=cyan,
}
\usepackage{tensor}
\usepackage{colortbl}
\usepackage{subcaption}
\usepackage{mwe}
\usepackage{color}
\graphicspath{{img/}}
\usepackage{bm}
\usepackage{url}
\usepackage{mathtools, nccmath}

\DeclarePairedDelimiter{\nint}\lfloor\rceil
\DeclareGraphicsExtensions{.eps}

\DeclareMathAlphabet{\mathbbmsl}{U}{bbm}{m}{sl}

\newtheorem{prop}{Proposition}
\begin{document}
\renewcommand{\baselinestretch}{0.95}

\title{Can robots mold soft plastic materials\\by shaping depth images?}
	
\author{Ege Gursoy, Sonny Tarbouriech, Andrea Cherubini%
\thanks{All authors are with LIRMM, Univ. Montpellier, CNRS, Montpellier, France. {\tt\small firstname.lastname@lirmm.fr}}
}

\bstctlcite{IEEEexample:BSTcontrol}

\markboth{}%
{Cherubini \MakeLowercase{\textit{et al.}}: Can robots mold soft plastic materials by shaping depth images?}

\maketitle

\begin{abstract}

Can robots mold soft plastic materials by shaping depth images? The short answer is no: current day robots can't. In this article, we address the problem of shaping plastic material with an anthropomorphic arm/hand robot, which observes the material with a fixed depth camera. Robots capable of molding could assist humans in many tasks, such as cooking, scooping or gardening. Yet, the problem is complex, due to its high-dimensionality at both perception and control levels. To address it, we design three alternative data-based methods for predicting the effect of robot actions on the material. Then, the robot can plan the sequence of actions and their positions, to mold the material into a desired shape. To make the prediction problem tractable, we rely on two original ideas. First, we prove that under reasonable assumptions, the shaping problem can be mapped from point cloud to depth image space, with many benefits (simpler processing, no need for registration, lower computation time and memory requirements). Second, we design a novel, simple metric for quickly measuring the distance between two depth images. The metric is based on the inherent point cloud representation of depth images, which enables direct and consistent comparison of image pairs through a non-uniform scaling approach, and therefore opens promising perspectives for designing \textit{depth image -- based} robot controllers. We assess our approach in a series of unprecedented experiments, where a robotic arm/hand molds flour from initial to final shapes, either with its own dataset, or by transfer learning from a human dataset. We conclude the article by discussing the limitations of our framework and those of current day hardware, which make human-like robot molding a challenging open research problem.

\end{abstract}


\IEEEpeerreviewmaketitle

\section{INTRODUCTION}

%\subsection{Motivation}

\IEEEPARstart{T}{hroughout} evolution, humans have assimilated amazing manipulation skills. An example is sculpting, which has progressed from prehistoric cave handprints to Renaissance masterpieces (see Fig.~\ref{Fig:sculpture}\footnote{Figure sources:  \url{https://en.wikipedia.org/wiki/File:GargasFlutings.jpg} and \url{https://www.michelangelo.org/images/artworks/moses.jpg}}). Achieving these results requires many capabilities. First, the sculptor must conceive a rough representation of the intermediate shapes, leading from initial to desired state of the molding material. Then, s/he must choose and use the most appropriate actions and tools for obtaining each of these intermediate shapes. While sculpting, s/he relies on perception (mostly visual and tactile) of the material properties and shape, to continuously update the process. Similar, albeit simpler, skills are needed for everyday tasks, such as gardening, cooking, massaging or manipulating cloak. 

%In~\cite{sculptNURB:07}, a novel robot sculpturing planning approach is
%presented for the 3D sculpturing with CAD-geometry data.
%It focuses on the approach of the automatic trajectory
%generation from the 3D CAD models of freeform surfaces
%to the machining path by the linear interpolation.
%
%the aim of~\cite{Pagliarini2009} is to try to summarize the main characteristics
%that might classify robot art as a unique and
%innovative discipline, and to track down some of the principles
%by which a robotic artifact can or cannot be considered
%an art piece in terms of social, cultural, and strictly
%artistic interest.

Recently,~\cite{Duenser20,MA2021150} have presented impressive robotic sculpting, achieved with a dual arm manipulator cutting clay via a hot wire. Other pioneer works on robotic sculpture include~\cite{sculptNURB:07,Pagliarini2009}. While their approaches are very original, the authors of these works focus on a specific task (sculpting), and rely on tools (e.g., hot wire) customized for their application and matter (e.g., clay). In the authors' view, molding in human-like manner soft materials would enable robots to perform a plethora of new operations and assist humans in many daily chores, well beyond robotic sculpture~\cite{Duenser20}-\cite{Pagliarini2009}. However, to date, roboticists have mainly focused on rigid object manipulation. The reason is that manipulating soft materials requires perceiving and modeling deformations at a high frame rate. While visual features can be consistently detected and tracked on rigid objects, they change over time on deformable materials, misleading both model-based and feature-based visual trackers. Tactile and force feedback could also be beneficial. These senses could complement the 3D geometry measured by vision, to infer the material's mechanical properties. Yet, complicated contact models are necessary to map force/tactile signals to the corresponding object surface displacements. Furthermore, the force/tactile measures should cover, with sufficiently high resolution, the whole robot skin. To the best of our knowledge, to date there exists no machine with such actuation and sensing capabilities.

\begin{figure}[t]
	\centering {\centering\includegraphics[width=\columnwidth]{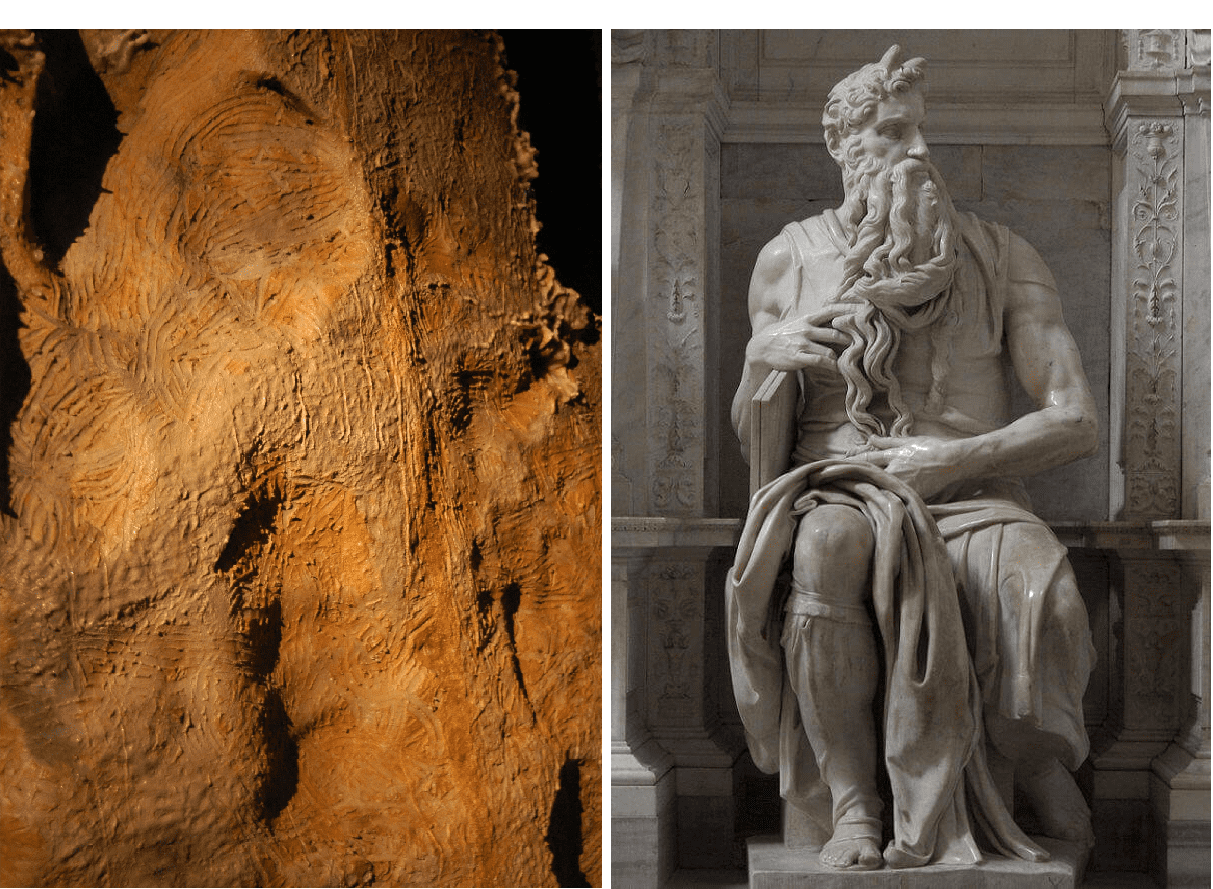}}
	\caption{It took over 20'000 years for human sculpture to evolve from the prehistoric cave finger flutings of Gargas in France (left) to Michelangelo's Moses (right).}
	\label{Fig:sculpture}
    \vspace{-0.7cm}
\end{figure}

In our previous work~\cite{cherubini:ijrr:2020}, we have addressed kinetic sand molding with 2D vision. Yet, without depth information, the set of controllable actions/effects is reduced. Typically, in~\cite{cherubini:ijrr:2020}, we could only modify the material's visible contours. That experience convinced us to opt, here, for depth vision, which -- compared to 2D vision -- enables shaping along all three directions. Besides, depth cameras, such as the Microsoft Kinect and the Intel RealSense, have recently gained a huge success in robotics.

Therefore, here, we explore the possibility of deforming soft plastic materials\footnote{Plastic materials remain deformed after the external force is removed.} by relying solely on depth images (i.e., with neither a model of the material, nor tactile feedback, nor registered point clouds). We propose the direct use of raw depth images instead of point clouds, to avoid cumbersome processing involving registration, and requiring high computational times and memory usage. Our rationale is similar to the one which motivated visual servoing directly in the image, rather than in the 3D space~\cite{chaumette:ram:2006}.

\subsection{Related Work}

Researchers have used point clouds to detect objects in indoor scenes~\cite{semanticLabel:13}, and to estimate their 6-DOF pose using 3D Geometric and Photometric feature descriptors~\cite{Hwang:12}. A related application is semantic object labelling, realized in~\cite{Engelmann:19} with deep learning. These techniques require rich datasets, such as the Yale-CMU-Berkeley real-life objects set~\cite{YCBdataset:17}, which is specifically designed for benchmarking manipulation research. Crowd-sourcing can be useful to build and label point cloud datasets, as shown in~\cite{sung2015robobarista}. Another application of point clouds is robot grasping, both with hands and suction cups. Pas and Platt~\cite{platt:18} propose to detect grasp points on novel objects presented in clutter, while accounting for the robot hand geometry. The authors of~\cite{dexnet:18} have introduced Dex-Net 3.0, a dataset of 2.8 million point clouds, suction grasps, and
grasp robustness labels, which computes the quality of the seal between suction cup and local object surface. Point clouds are studied well beyond robotics, with applications including plant phenotyping~\cite{phenotype:19}.

Despite this success, point cloud dimensions make processing cumbersome: even most current day implementations of the Iterative Closest Point (ICP) algorithm~\cite{icp:94} for aligning two point clouds are too slow for robot feedback control. ICP has been originally designed and applied to rigid scenes~\cite{Fioraio2011RealtimeVA}, with improved versions, as in~\cite{Liu2018EfficientGP}, where translation and rotation are decoupled, and a fast Bound and Branch algorithm globally optimizes the 3D translation parameter first. Recently,~\cite{carlone:21} has proposed a fast (milliseconds) and certifiable algorithm for the registration of two point clouds in the presence of numerous outlier correspondences. The authors of~\cite{Ma:19} have extended point registration to non-rigid clouds, via a robust transformation learning scheme. The principle is to iteratively establish point correspondences and learn the nonrigid transformation between two given sets of points. To this end, they cast the point set registration into a semi-supervised learning problem.

At a higher level, many efforts have been put into non-rigid object tracking. DynamicFusion~\cite{newcombe:15} is a dense simultaneous localization and mapping framework capable of reconstructing non-rigidly deforming scenes in real-time, by
fusing RGB-D images. Since it does not require a template or other prior scene model, the approach is applicable to a wide range of moving objects and scenes.
Other model-free approaches are: ~\cite{khalil2010visual}, which tracks -- using touch and vision -- the surface of a deformable object to be grasped by a robot, and~\cite{staffa2015segmentation}, where a neural network can visually segment a pizza dough. An elegant alternative~\cite{abbeel:13} consists in including -- if they are known -- the physical properties of the object, to track it from a sequence of point clouds, with a probabilistic generative model.

When it comes down to shaping the object, the robot should be capable of predicting the next state given the current state and a proposed action. To this end,~\cite{elliott2018robotic} use some defined primitive tool action for rearranging dirt and~\cite{schenck2017learning} presents a Convolutional Neural Network for scooping and dumping granular materials. The authors of~ \cite{li2018learnParticle} learn a particle-based simulator for complex control tasks; this simulator adapts to new environments or to unknown dynamics within few observations.

\subsection{Contributions}

In this paper, we consider a manipulator equipped with an anthropomorphic hand and capable of applying three different actions at various positions of some plastic material, while observing it through a fixed depth camera. The robot must mold the material into a desired shape, by finding and applying the correct sequence of actions at the right positions. To find the correct sequence, the robot should predict the effect of each action on the environment, which is represented by a depth image. We rely on some realistic hypotheses, to make the prediction problem tractable, so that it breaks down to replacing local rectangular patches within the depth image. We generate the patches consequent to action execution, with three alternative data-based methods, two with \textit{encoder-decoder image generator network}, and a simpler one with \textit{difference compensation}. We train all three prediction methods with a robot dataset and with a human dataset, and we design a novel scalar metric for assessing the methods' precision, in terms of distance between point clouds. We assess our method in a series of experiments, where a robot molds desired shapes, either with its own dataset, or by \textit{transfer learning} from the human dataset. Finally, we show that our method for finding the correct sequence of actions could generalize across actions, materials, and geometry.

The contributions of our paper are the following:
\begin{enumerate}
	\item 
	We propose an \textit{original framework for depth image -- based robot shaping of un-modeled plastic material}. Given an initial and a desired point cloud of the material, the robot plans and executes a sequence of different actions at various positions to mold the material. Due to its high dimensionality, the problem of predicting posterior (consequent to robot actions) point clouds, without a model of the material, would be intractable without an original idea: we show that, under reasonable assumptions, the shaping problem can be mapped from point cloud to depth image space, with the following benefits.
    \begin{itemize}
       \item
        Point cloud registration and matching are not required anymore,
	  \item 
        the lower needs in terms of computation time and memory usage make the action effect prediction and action planning sub-problems tractable,
        \item
        standard image processing techniques as well as deep learning can be applied directly to the depth images.
    \end{itemize}
    \item
     To attain depth image -- based shaping, we define a scalar metric for measuring the \textit{distance between two point clouds}, in depth image space. The proposed distance metric is based on the observation that depth images inherently encode a point cloud, allowing for the direct comparison of image pairs in a consistent manner. In our work, we use this metric: \textit{(a)} as the loss function of the generator which models changes produced by a robot action, \textit{(b)} as a measure for planning the sequence of actions leading from initial to final point clouds, and \textit{(c)} to assess the success of a point cloud shaping task. Since it is the weighted mean value of the difference between two 2D images, this metric can be computed very quickly, and with low memory needs, opening promising perspectives for designing depth image -- based robot controllers. 
	\item
	We assess the framework in a \textit{series of unprecedented experiments} with a robotic arm + hand, molding flour from initial to final shapes, by relying on both a robot and a human dataset. By exploiting transfer learning from the human, we do not need numerous robot experiments as in~\cite{schenck2017learning}. Furthermore, in contrast with similar works ~\cite{elliott2018robotic,schenck2017learning,li2018learnParticle}, by using a anthropomorphic robot, we push robotics one step closer to human-like molding. Yet, robotic molding remains an open problem, due to the many limitations of current day robots, emphasized in our experiments.  
\end{enumerate}

%IN CASE CONTRIBUTIONS WRT ICP ~\cite{carlone:21} rigid
%~\cite{newcombe:15} to generate loss function
%~\cite{Ma:19} learning

\section{Modeling Depth Images and Point Clouds}

In this section, we recall the fundamental notions of depth image theory, required in our work.
 
The \textit{luminance} of each pixel of a depth camera of resolution $w \times h$ and quantification $b$ bits is:
\begin{equation}
L_{uv} \in \mathcal{L} = \left\{0, 1, \dots, 2^b-1\right\} \subset \mathcal{N}
\end{equation}
with $\left(u,v\right) \in \mathcal{I} = \left\{1, \dots, w\right\} 
\times \left\{1, \dots, h\right\} \subset \mathcal{N}^2$ the pixel coordinates. The entire depth image is the matrix: 
\begin{equation}
\mathbf{L} = \left(
\begin{array}{ccc}
L_{11} & \cdots & L_{1w} \\
\vdots  & \ddots & \vdots \\
L_{h1} & \cdots & L_{hw}
\end{array}
\right) \in \mathcal{L}^{w \times h}. 
\end{equation}
In the following, we refer to \textit{depth images} simply as \textit{images}.

We express point coordinates in meters in the \textit{camera frame} (center at the camera optical center, $X$ increasing with the pixel columns, $Y$ increasing with the pixel rows, $Z$ coincident with the optical axis, and positive for points in front of the camera).
Assuming a linear depth model, for any point $\mathbf{p} = \left(X Y Z\right)^\top$ seen in the depth image at $\left(u,v\right)$, luminance $L_{uv}$ can be converted to $Z$:
\begin{equation}
Z \left( L_{uv} \right) = \frac{Z_{min}-Z_{max}}{2^b-1} L_{uv} + Z_{max} \in \left[Z_{min}, Z_{max}\right].
\label{eq:depth}
\end{equation}
This requires knowing the camera depth range $\left[Z_{min}, Z_{max}\right] \subset \mathcal{R}^+$.
Note that this model clips at $Z_{max}$ (respectively,  $Z_{min}$) the depths of all points which have depth greater (smaller) than $Z_{max}$ ($Z_{min}$); the corresponding pixels are black (white), with $L_{uv}=0$ ($L_{uv}=2^b-1$).
We consider un-distorted perspective projection (i.e., a pinhole camera model). Then, the two other camera frame coordinates of $\mathbf{p}$ can be derived from $Z$ as:
\begin{equation}
X \left( u \right)  = \frac{u-u_0}{f_x} Z \text{  and  } Y \left( v \right) = \frac{v-v_0}{f_y}  Z
\label{eq:pinhole}
\end{equation}
with $\left(u_0,v_0\right) \in \left[1, w\right] 
\times \left[1, h\right] \subset \mathcal{R}^2$ the camera principal point coordinates in the image plane and $\left( f_x, f_y \right) \in \mathcal{R}^{2}$ the camera focal lengths expressed in pixels. All points visible in the depth image belong to the camera \textit{viewing frustum} $\mathcal{F} \subset \mathcal{R}^3$, which is a truncated pyramid lying between $Z_{min}$ and $Z_{max}$.

Let us define $\mathbf{m}$ the vector function which combines~(\ref{eq:depth}) and~(\ref{eq:pinhole}) to map a pixel and its luminance to the camera frame coordinates of the corresponding point:
\begin{equation}
\begin{array}{rccc}
\mathbf{m}:  & \mathcal{I} \times \mathcal{L}  & \rightarrow & \mathcal{F} \vspace{.2cm}\\
& \left( 
\begin{array}{c}
u\\
v\\
L_{uv} 
\end{array}
\right) & \mapsto & 
\left( 
\begin{array}{c}
X \left( u \right)\\
Y \left( v \right)\\
Z \left( L_{uv} \right)
\end{array}
\right)
\end{array}
\end{equation}
Note that  this mapping is \textit{bijective}: 
\begin{equation}
\begin{array}{rccc}
\mathbf{m}^{-1}:  & \mathcal{F}  & \rightarrow & \mathcal{I} \times \mathcal{L} \vspace{.2cm}\\
& \left( 
\begin{array}{c}
X\\
Y\\
Z 
\end{array}
\right) & \mapsto & 
\left( 
\begin{array}{c}
f_xX / Z + u_0\\
f_yY / Z + v_0\\
\left(2^b-1\right) \frac{Z - Z_{max}}{Z_{min} - Z_{max}}
\end{array}
\right).
\end{array}
\label{eq:PtoI}
\end{equation}
This is not the case of standard cameras, which cannot measure the point depth.

\textit{\underline{Definition: visible point cloud.}} Let us now represent the depth image as a set $\mathcal{D}$ of $w\times h$ triplets, containing the coordinates of each pixel of $\mathbf{L}$ along with its luminance:
\begin{equation}
\mathcal{D} = \left\{ \left(  u_i, v_i, L_{u_iv_i} \right)^\top, i = 1, \dots, w\times h \right\} \subset \mathcal{I} \times \mathcal{L}.
\end{equation}
The visible point cloud $\mathcal{P} = \left\{ \mathbf{p}_1, \dots \mathbf{p}_{w\times h}\right\} \subset \mathcal{F}$ corresponding to $\mathbf{L}$ is the \textit{image} (in mathematical terms) of $\mathcal{D}$ under $\mathbf{m}$:
\begin{equation}
\begin{array}{rccc}
\mathbf{m}:  & \mathcal{D}  & \rightarrow & \mathcal{P} 
\end{array}
\end{equation}
Note that $\dim \mathcal{P} = w \times h$. Note also that not all points of $\mathcal{P}$ exist physically. In particular, points corresponding to $L_{uv}=0$ and to $L_{uv}=2^b-1$ are artefacts due to the limited camera range. Furthermore, two points in $\mathcal{P}$ cannot be on the same optical ray (or else, one will occlude the other):
\begin{equation}
\forall 
\left( \mathbf{p}_i, \mathbf{p}_j \right) \in \mathcal{P}: \nexists 
k \in \mathcal{R}:  \mathbf{p}_i = k \mathbf{p}_j.
\end{equation}
In this paper, we will also use normalized pixel coordinates:
\begin{equation}
x = \frac{u-u_0}{f_x} \;\;\;\text{and}\;\;\; y = \frac{v-v_0}{f_y},
\label{eq:camCoord}
\end{equation} 
to rewrite equation~(\ref{eq:pinhole}) as:
\begin{equation}
X = x Z \;\;\; \text{and} \;\;\; Y = y Z.
\label{eq:pinhole2}
\end{equation}

%We name
%\begin{equation}
%\mathcal{P} = \left[
%\begin{array}{ccc} 
%\mathcal{P}_1 & \mathcal{P}_2 & \dots\\
%1 & 1 & \dots
%\end{array}
%\right] 
%\end{equation} 
%a point cloud representing the environment. It is possible to extract each point $\mathbf{p_i} = \left[X_i Y_i Z_i\right]^\top$ from $\mathcal{P}$ via the mapping:
%\begin{equation}
%\mathcal{P}_i \left( \mathcal{P} \right) = \left[
%\begin{array}{cccc} 
%1 & 0 & 0 & 0\\
%0 & 1 & 0 & 0\\
%0 & 0 & 1 & 0  
%\end{array}
%\right] \mathcal{P}
%\left[
%\begin{array}{c}
%0_1\\
%\vdots\\
%0_{i-1}\\
%1\\
%0_{i+1}\\
%\vdots\\
%0_{n}
%\end{array}
%\right],
%\end{equation} 
%and to represent the point cloud by set $\mathcal{P} = \left\{ \mathcal{P}_1, \dots\right\}$. 
%
%First: all of the points are within the camera's view frustum: 
%\begin{equation}
%\begin{array}{l}
%\forall \mathcal{P}_i \in \mathcal{P}: \left\{
%\begin{array}{l}
%X_i \in \left[\frac{1-u_0}{f_x} Z; \frac{w-u_0}{f_x} Z \right]\\
%Y_i \in \left[\frac{1-v_0}{f_y} Z; \frac{h-v_0}{f_y} Z \right]\\
%Z_i \in \left[Z_{min}; Z_{min} \right],
%\end{array}
%\right.
%\end{array} 
%\end{equation}

\section{Problem statement}

\subsection{The point cloud shaping task}
\label{sect:shapePointCloud}

Consider a \textit{robot}, observing via a fixed depth camera the \textit{environment}. The robot can execute -- on the point cloud $\mathcal{P}$ representing the environment -- an action $\mathbf{a}$.

\begin{table*}[t!]
	\caption{Actions used in this work, with their characteristics.} 
	\label{table:Act}
	\centering
	\begin{tabular}{|C{1.2cm}|C{3.0cm}|C{3.9cm}|C{2.2cm}|C{1.4cm}|C{1.4cm}|C{1.4cm}|}
		\cline{1-7}
		\textbf{Action $\mathbf{a}^i$} & \textbf{Description} & \textbf{Illustration} & \textbf{Effect on $\mathbf{L}$} & \textbf{$\Delta X$ (mm)} & \textbf{$\Delta Y$ (mm)} & \textbf{$\Delta Z$ (mm)} \\ 
		\hline
		Grasp $\mathbf{a}^{g}$  & The hand initially in claw shape is closed. &
		\vspace{0.3mm}
		\includegraphics[width=0.07\textwidth]{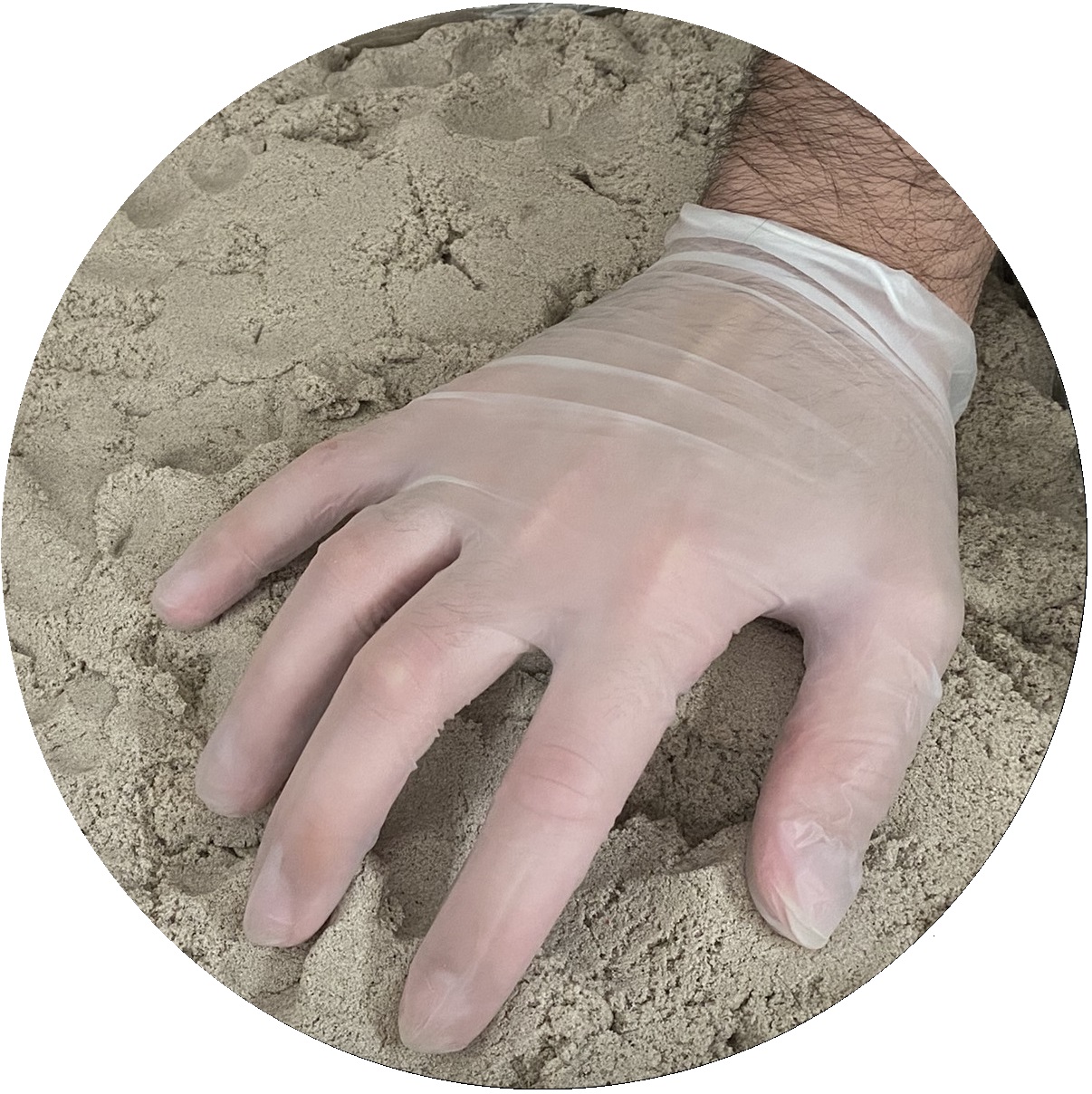}
		\includegraphics[width=0.07\textwidth]{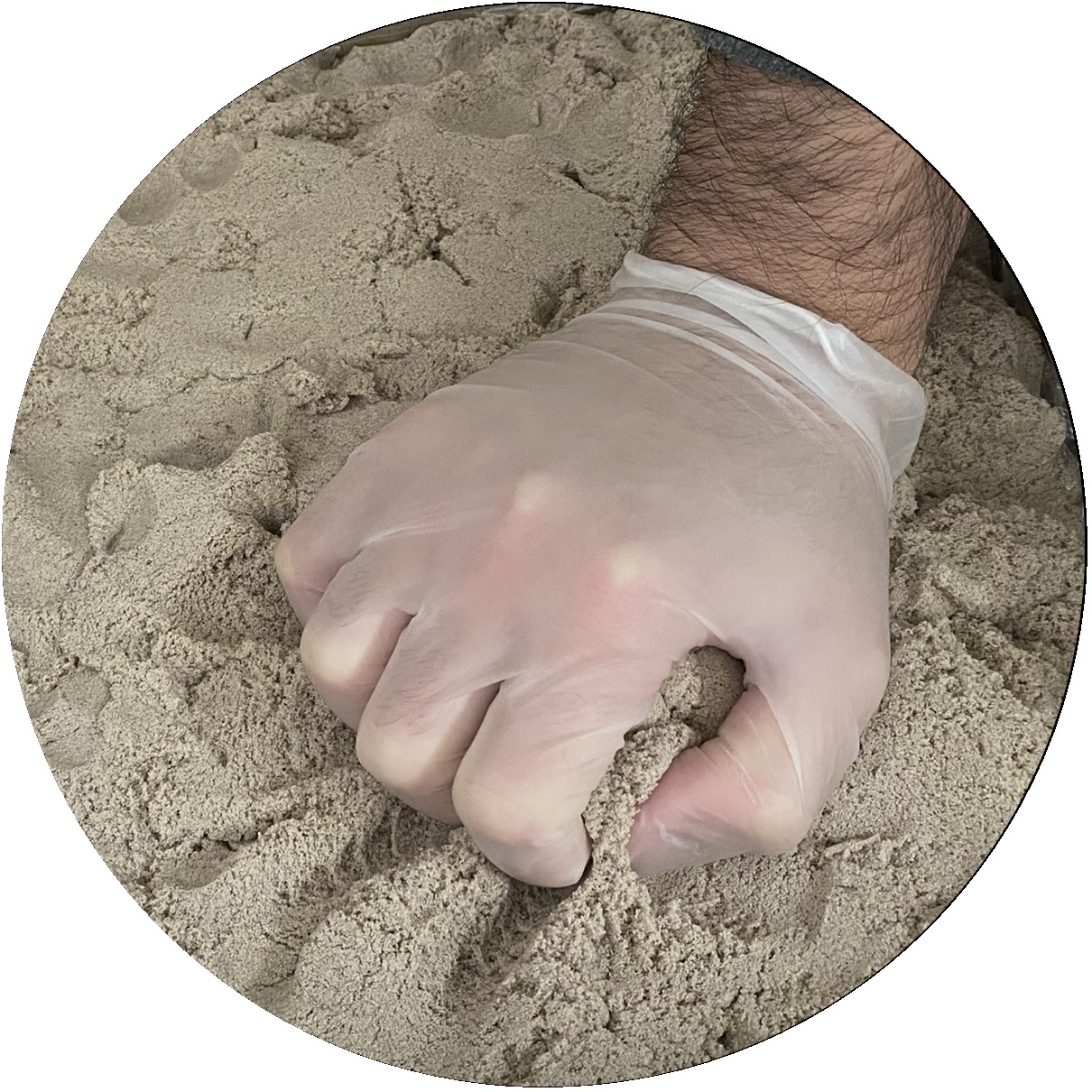} & \includegraphics[width=0.1\textwidth]{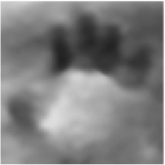}
		& 218 & 195 & 50 \\
		\hline
		Knock $\mathbf{a}^{k}$ & The closed fist pushes downward.  & 
		\includegraphics[width=0.07\textwidth]{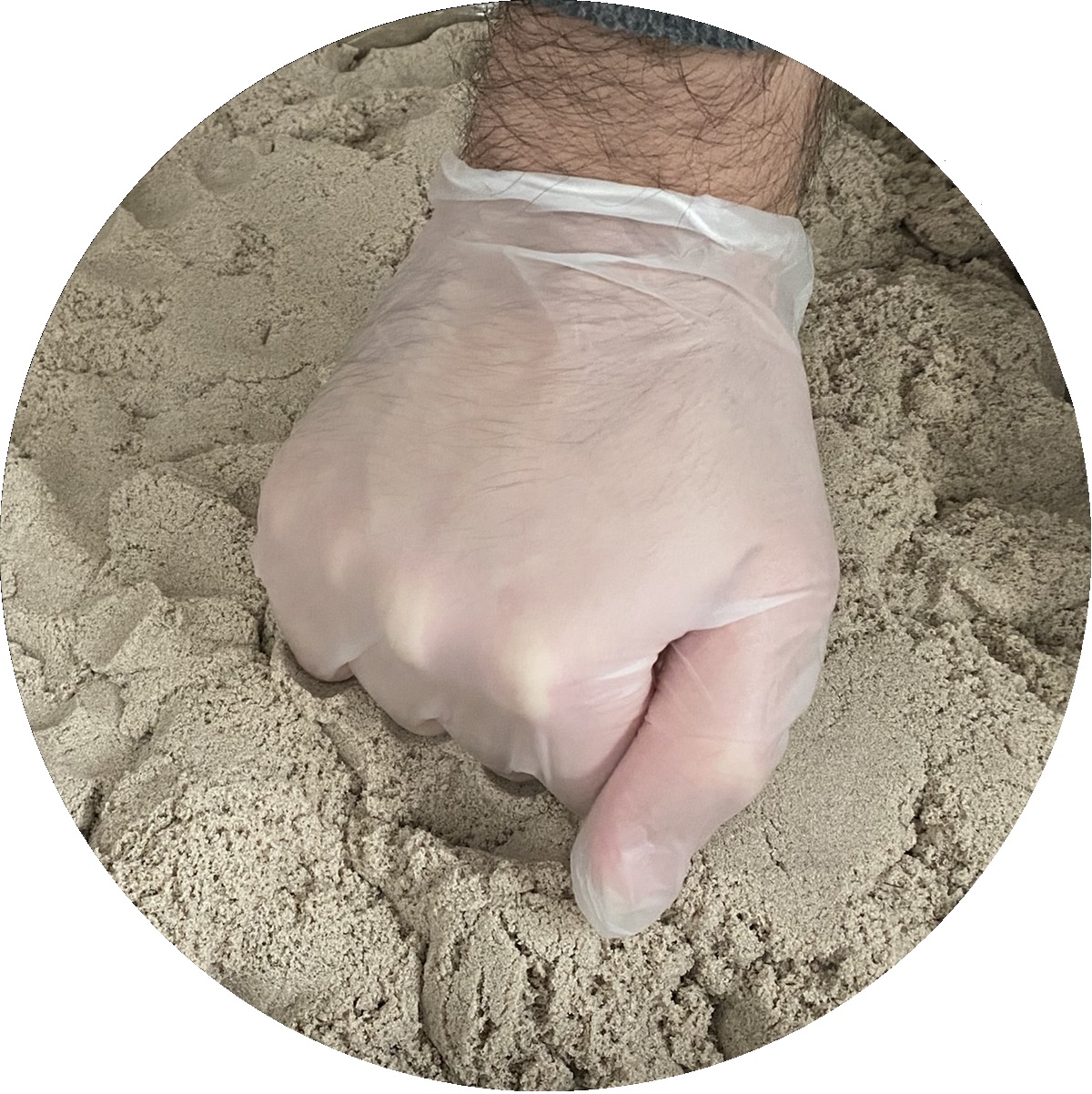}
		\includegraphics[width=0.07\textwidth]{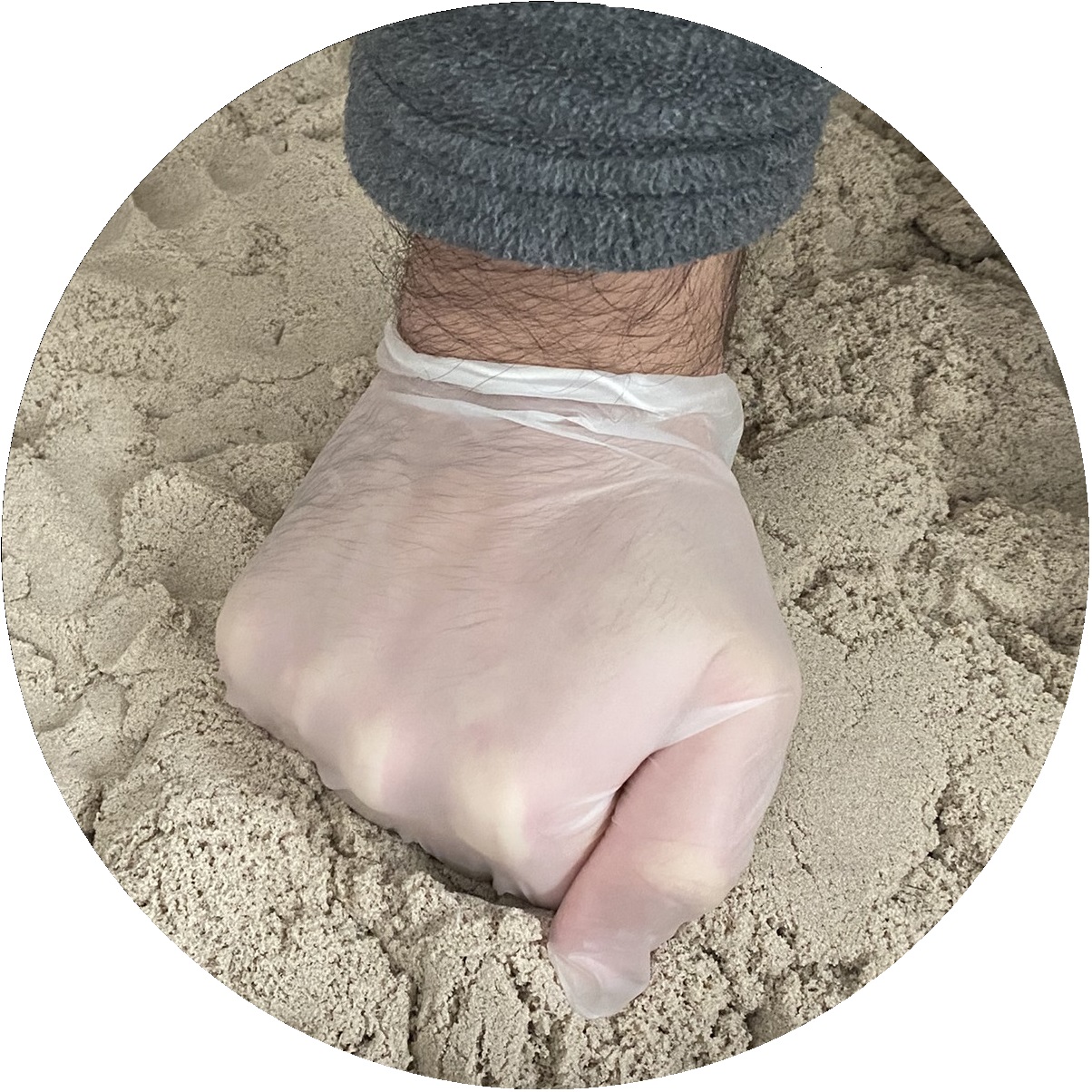} & \includegraphics[width=0.1\textwidth]{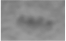} &  84 & 75 & 50\\
		\hline
		Press $\mathbf{a}^{r}$ & The fist bottom pushes downward.  & 
		\includegraphics[width=0.07\textwidth]{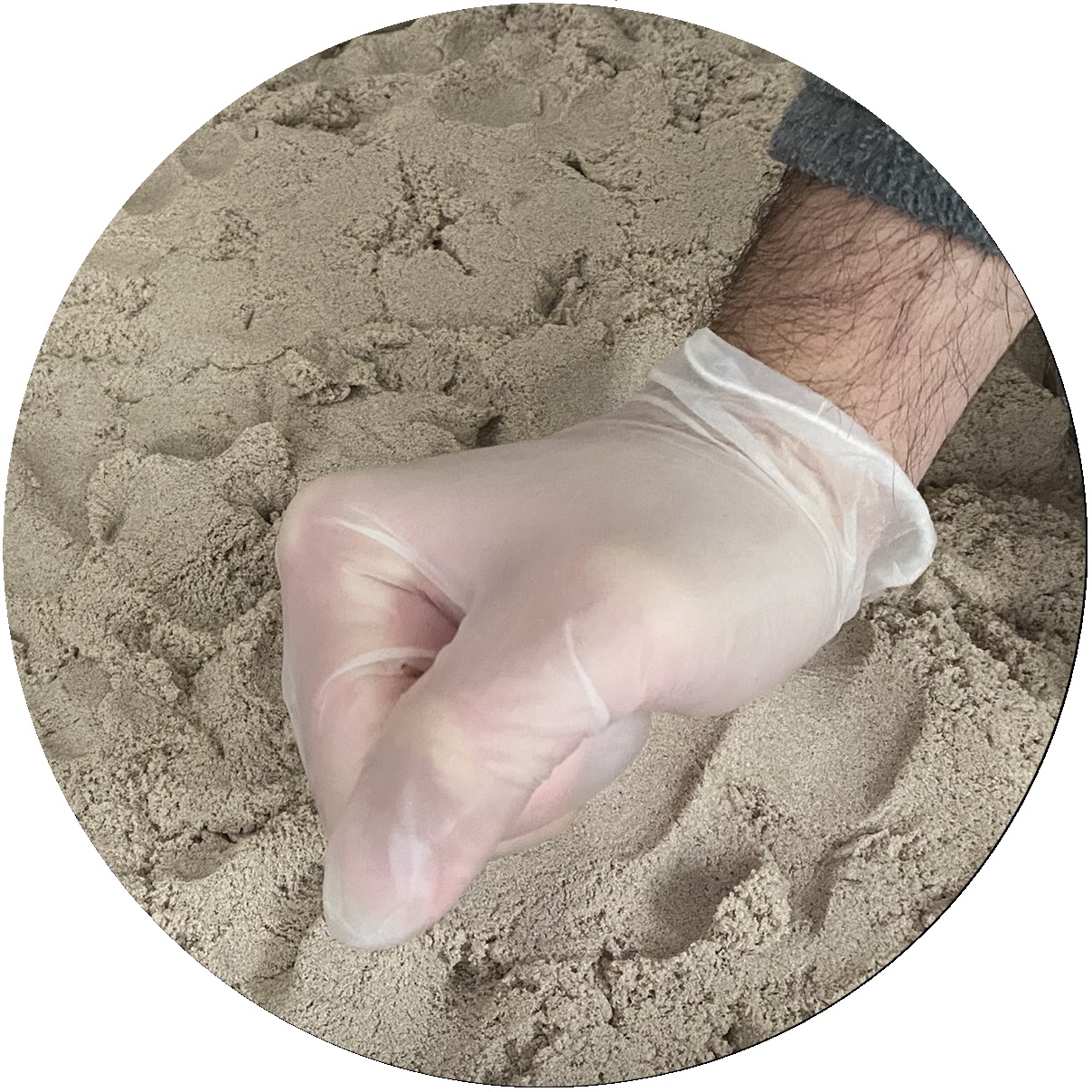}
		\includegraphics[width=0.07\textwidth]{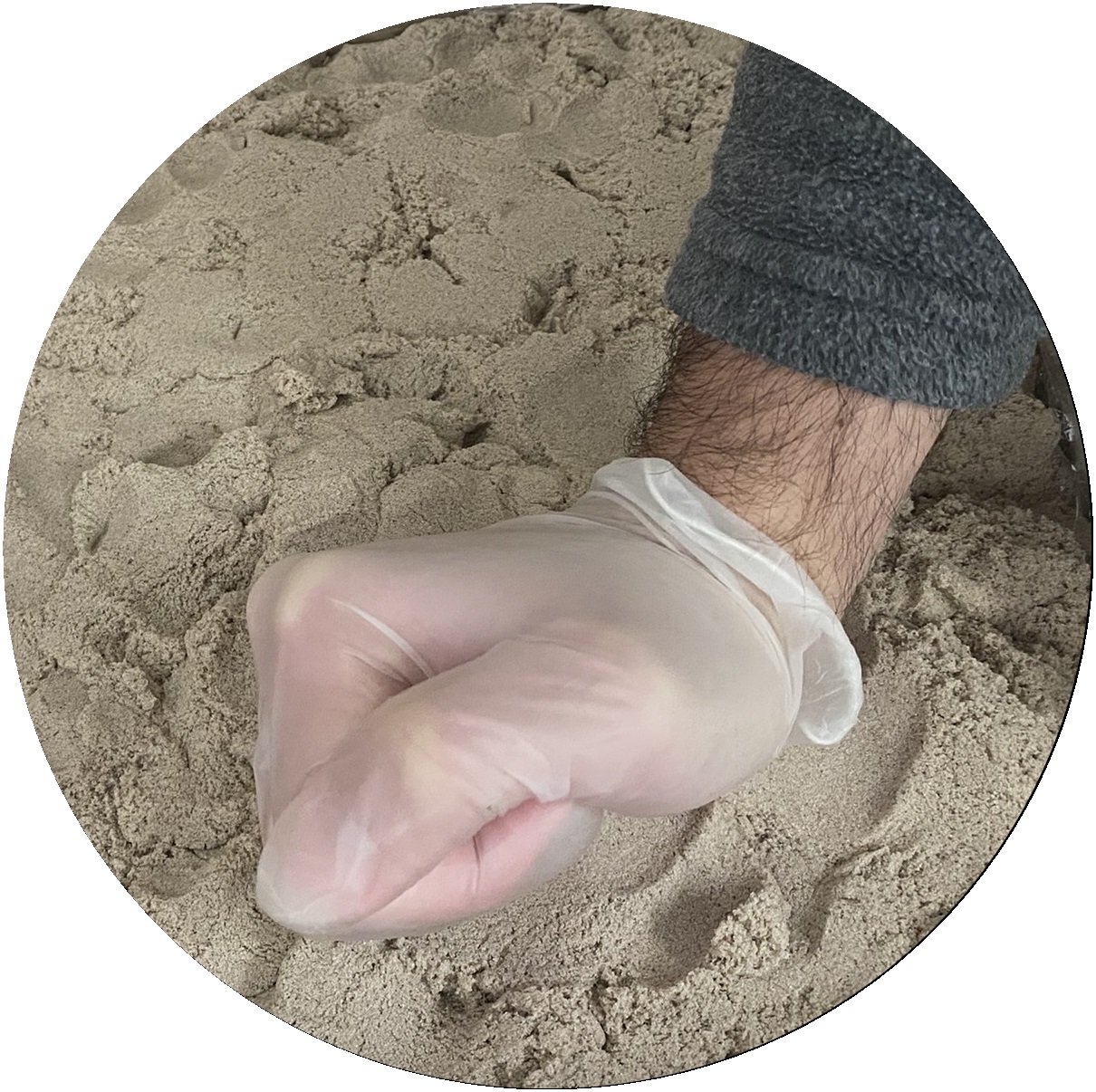} & \includegraphics[width=0.045\textwidth]{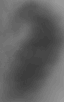} &  46 & 103 & 50\\
    	\hline
		Pinch $\mathbf{a}^{n}$ & The index and middle fingers close. & \includegraphics[width=0.07\textwidth]{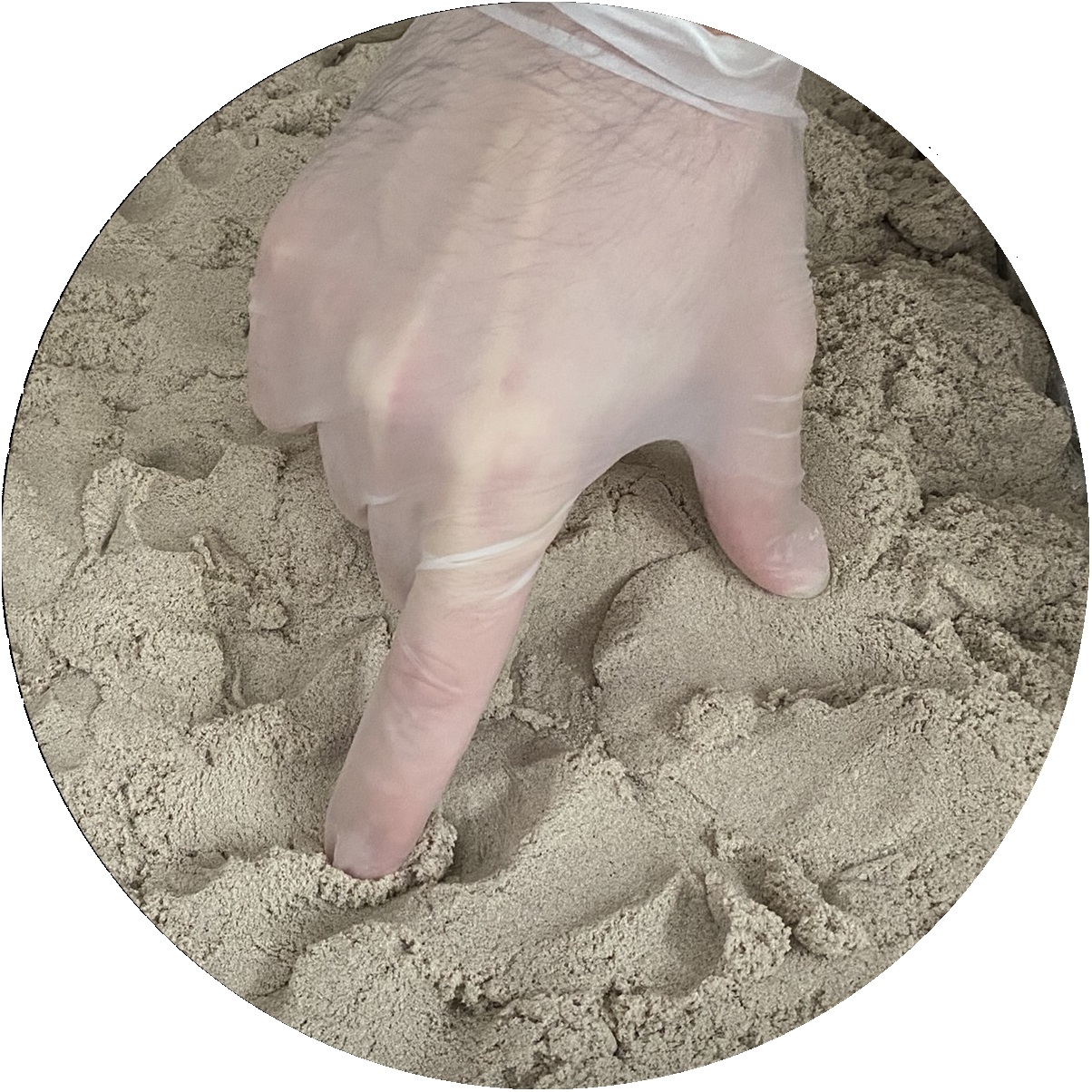}
		\includegraphics[width=0.07\textwidth]{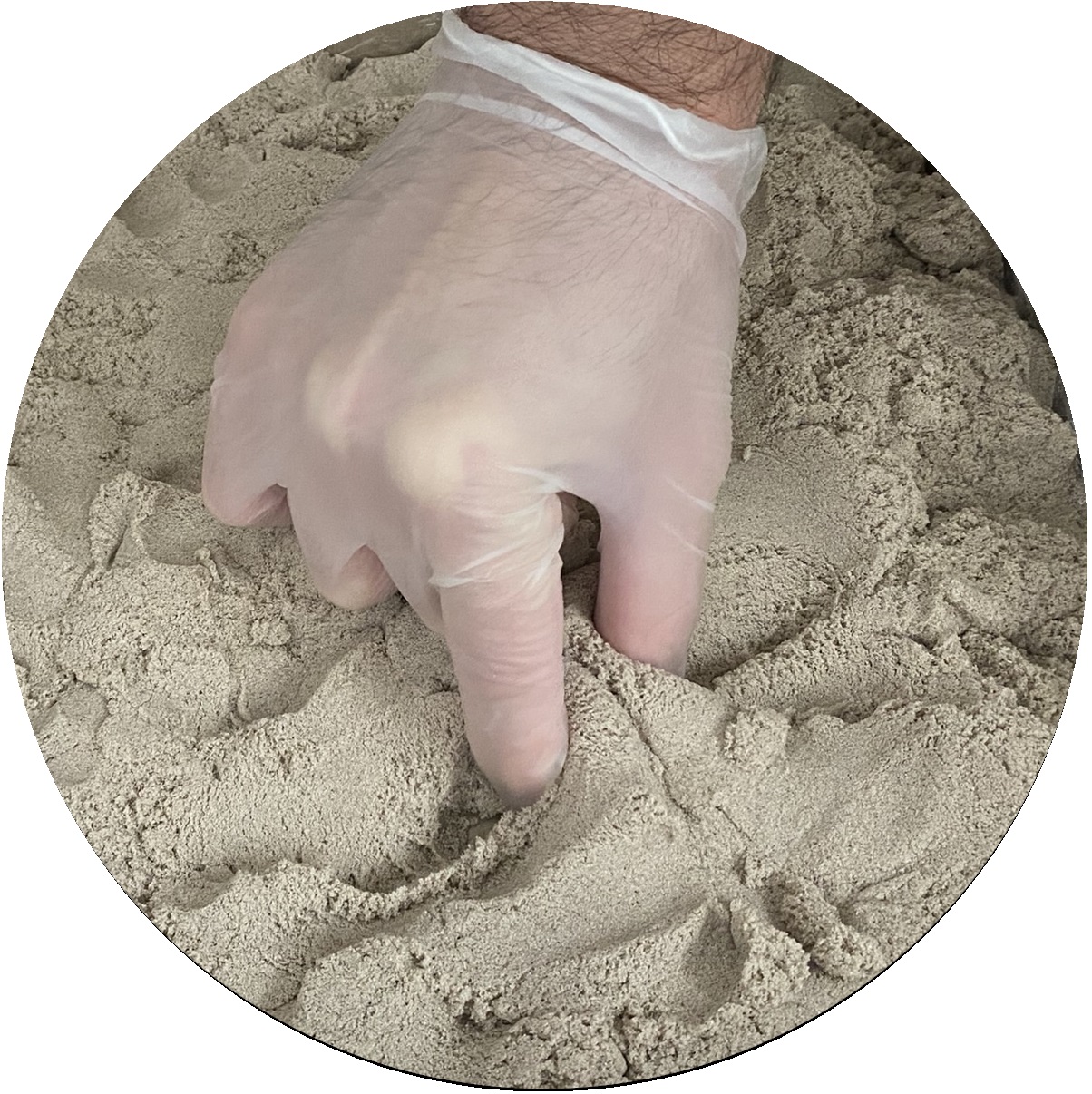} & \includegraphics[width=0.04\textwidth]{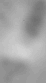} &
		42 & 85 & 50\\ 
		\hline
		Poke $\mathbf{a}^{p}$ & The index finger pushes downward. & \includegraphics[width=0.07\textwidth]{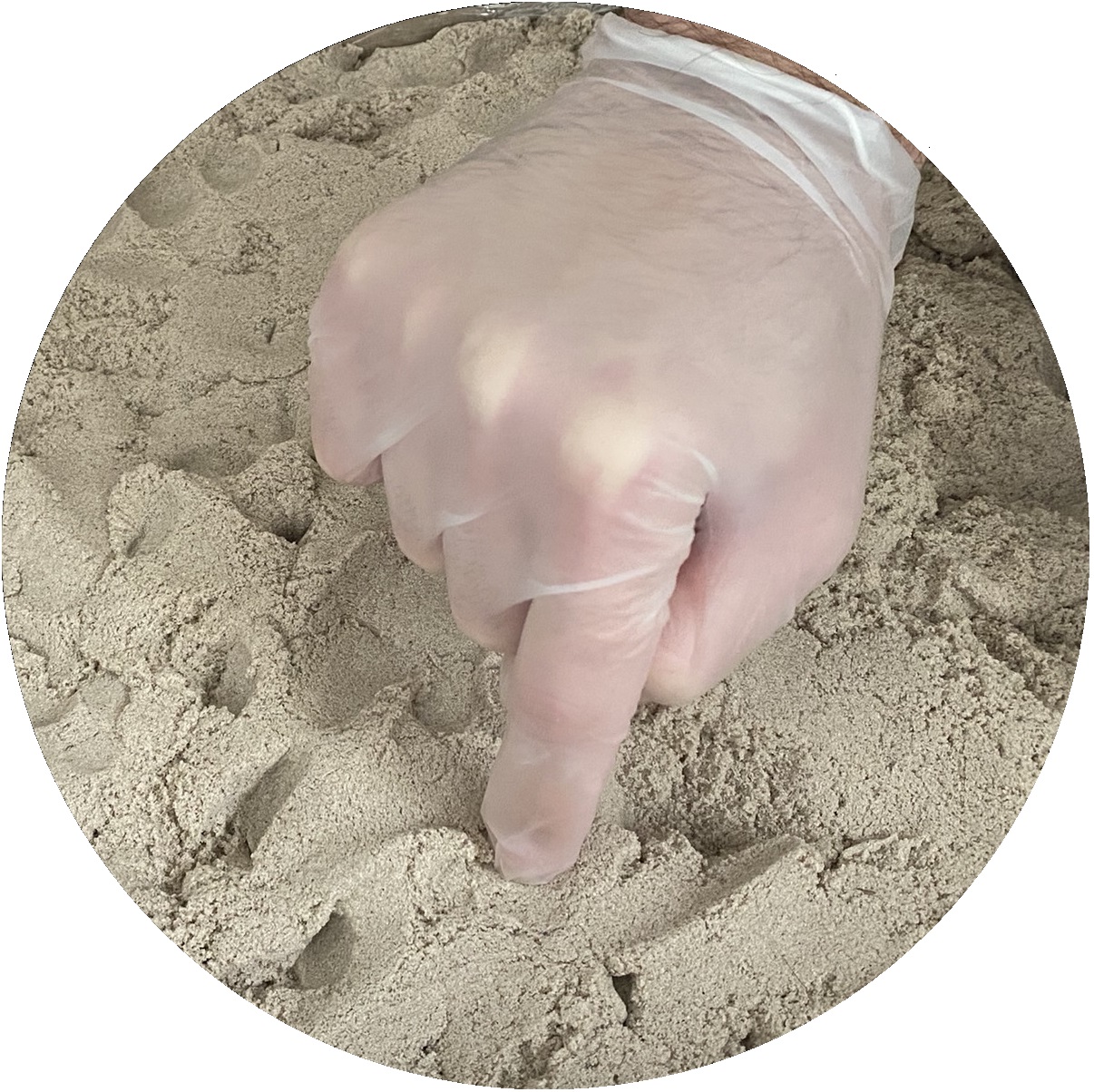}
		\includegraphics[width=0.07\textwidth]{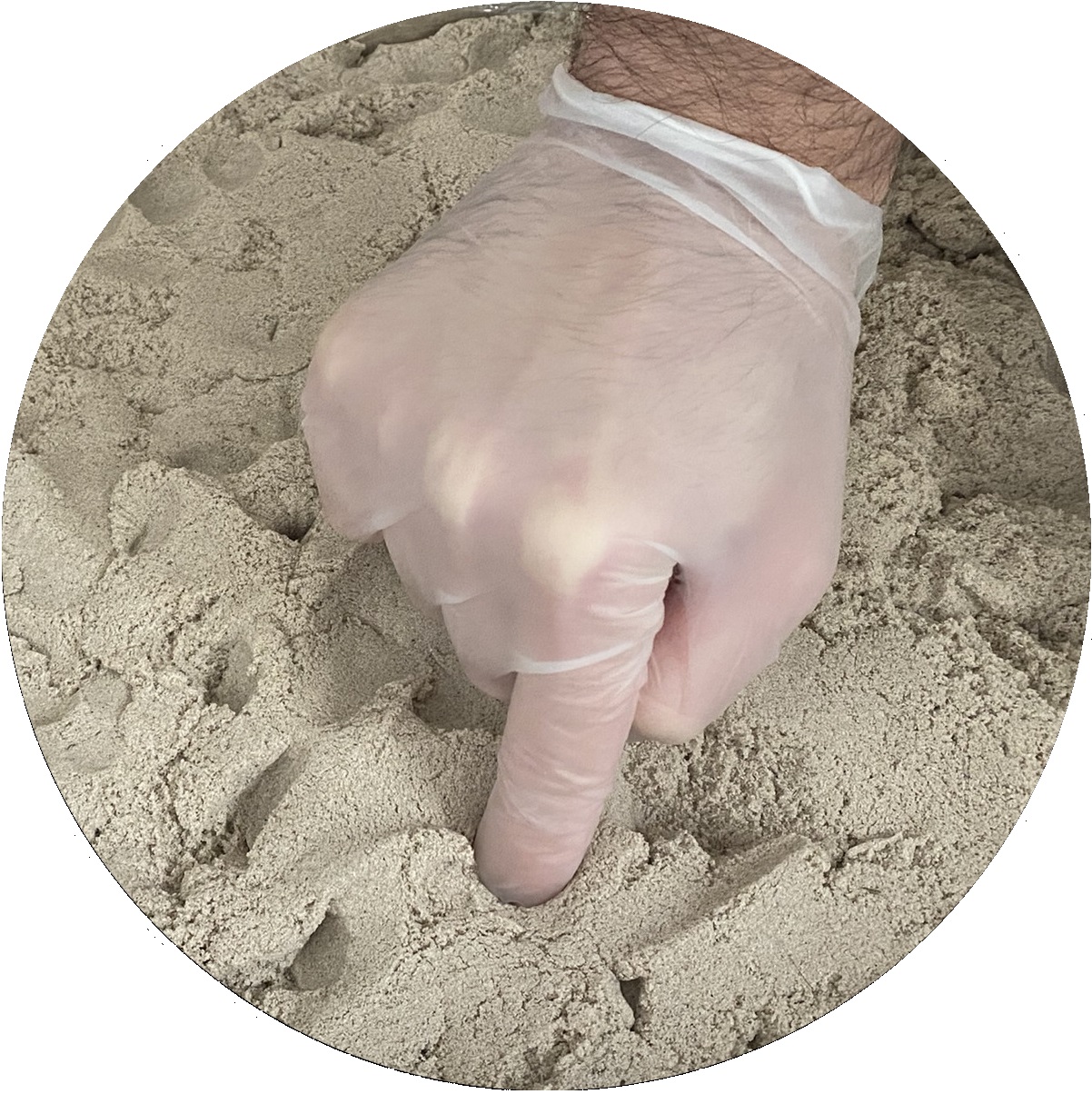} & \includegraphics[width=0.06\textwidth]{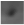} &
		22 & 18 & 50\\ 
		%\hline
		%Tap $\mathbf{a}^{t}$ & push towards the sand with open palm & && & &\\ 
		\hline
	\end{tabular}
 \vspace{-0.4cm}
\end{table*}

\textit{\underline{Definition: Robot action $\mathbf{a}$.}} An action is a set of joint space trajectories, which makes the robot modify $\mathcal{P}$. Each action $\mathbf{a}^i$ is applicable at a position $\mathbf{t} = \left(X Y Z\right)^\top$ in the camera frame, within a finite set of possible (action-dependent) positions, $\mathbf{t}^{j}$, $j =1,\dots,N$. We name ${\mathcal{A}} = \left\{ \mathbf{a}^{ij} \right\}$ the set of all actions over all positions.

In this work, we have selected $5$ actions (grasp, knock, press, pinch and poke) after having inspected the database of human sculpting images presented in~\cite{cherubini:ijrr:2020}. Table~\ref{table:Act} shows the characteristics of each of these actions: name, description, illustration (when applied by humans), effect on a depth image $\mathbf{L}$, and ``size'' (parameters $\Delta X$, $\Delta Y$, $\Delta Z$, which we will detail hereby).

We model the robot and environment (i.e., the point cloud) as a discrete-time system. At each iteration $k \in \left[  1, K \right] \subset \mathcal{N}$, the robot executes an action $\mathbf{a}_k \in {\mathcal{A}}$ on point cloud $\mathcal{P}_k$. The joint space trajectories of the action are completed within the iteration, and the point cloud is changed to a new point cloud
\begin{equation}
{\mathcal{P}}_{k+1} = \mathbf{f} \left({\mathcal{P}}_{k}, \mathbf{a}_k \right),
\label{eq:model}
\end{equation}
with $\mathbf{f} \left( . \right)$ an appropriate mathematical function modeling the point cloud changes.

\textit{\underline{Hypotheses}} We rely on the following assumptions:
\begin{enumerate}
	\item 
	The process is Markov; therefore, we can remove the dependency from $k$ and rewrite~(\ref{eq:model}) as:
	\begin{equation}
	{\mathcal{P}}_{\mathbf{a}} = \mathbf{f} \left({\mathcal{P}}, \mathbf{a} \right).
	\label{eq:markovModel}
	\end{equation}
 This assumption also implies that: 1/ no deforming phenomena, aside from those caused by $\mathbf{a}$, act on the point cloud, and 2/ the deforming phenomena caused by $\mathbf{a}$ are constant.
	%\item
%both the prior and posterior point clouds are visible
	\item
	The effect of each action on the point cloud is bounded in the 3D space, and contained within a box, denoted $\mathcal{B}_{ij}$ for action $\mathbf{a}^{ij}$. We have designed the actions in ${\mathcal{A}}$, so that all boxes $\mathcal{B}_{ij}$ are entirely visible (i.e., $\mathcal{B}_{ij} \subset \mathcal{F}$).
	We formalize this hypothesis as follows: the points which differ between the prior and posterior point clouds are:
	\begin{equation}
	\begin{array}{l}
	\mathbf{p} \in \left( \mathcal{P} \cup \mathcal{P}_{\mathbf{a}_{ij}} \right) \setminus \left(\mathcal{P} \cap \mathcal{P}_{\mathbf{a}_{ij}} \right)\\
	\implies \mathbf{p} \in \mathcal{B}_{ij} = \left[X^- ; X^+ \right] \times \left[Y^- ; Y^+ \right] \times \left[Z^- ; Z^+ \right],
	\end{array}
	\label{eq:bbox}
	\end{equation}
	with $\left(X^\pm Y^\pm Z^\pm\right)^\top = \mathbf{t} \pm \left( \Delta X \Delta Y \Delta Z \right)^\top$. The box is centered at the action's position $\mathbf{t}^{j}$, and its size depends on the action type (e.g., an action may affect a bigger volume than another, see Table~\ref{table:Act}).
	\item
	The positions of all actions in ${\mathcal{A}}$ have the same $Z$ coordinate, and the depth of $\mathcal{B}_{ij}$ is negligible with respect to the action depth ($\Delta Z << Z$), so $Z^- \approx Z^+ \approx Z$.
	\item
	The effects of the actions are position invariant: given $i$, the effect of $\mathbf{a}^{i}$ on the points contained in $\mathcal{B}_{ij}$ is the same for all $\mathbf{t}_{j}$, $j =1,\dots,N$.
\end{enumerate}

\textit{\underline{Definition: point cloud shaping task}} Given a final desired point cloud $\mathcal{P}_*$, \textit{the point cloud shaping task} consists in shaping point cloud $\mathcal{P}$ into $\mathcal{P}_*$. We define $d_k \left(\mathcal{P}_k, \mathcal{P}_*\right) \in \mathcal{R}^+$ as a scalar function, which measures the distance between $\mathcal{P}_k$ and $\mathcal{P}_*$, so that $d_k = 0 \iff \mathcal{P}_k = \mathcal{P}_*$. Then, the task of shaping a point cloud to $\mathcal{P}_*$ with an accuracy $\bar{d} \geq 0$, consists in applying a finite sequence of actions within ${\mathcal{A}}$: $\mathbf{s}_*= \left\{ \mathbf{a}_1, \dots, \mathbf{a}_K \right\}$ so that after $K$ iterations: $d_K  \leq \bar{d}$. We tolerate such an upper bound on the accuracy, since even a human would be incapable of perfectly reproducing ($d_K  = 0$) a given point cloud. Yet, the sequence of actions should make $d_K$ decrease. More formally, it must be possible to apply, at each iteration $k$, an action $\mathbf{a}_k$ to reduce $d_k$:
\begin{equation}
d_{k+1} < d_k.
\label{eq:errorDecrease}
\end{equation}
The above equation assumes that the task can be solved using a greedy algorithm. Greedy heuristics are known to produce suboptimal results on many problems, and longer term planning should be preferred. Yet, since the focus of this paper is not planning, we have posed this assumption to simplify the point cloud shaping implementation. Below, we explain how we have defined $d$.
 
\subsection{Distance between two Visible Point Clouds}

Consider two point clouds $\mathcal{P}$ and $\mathcal{Q}$, visible by the depth camera in images $\mathbf{L}_P$ and $\mathbf{L}_Q$.
In this Section, we define the metric $d \left( \mathcal{P}, \mathcal{Q} \right) \in \mathcal{R}^+$ used in our work to measure the distance between $\mathcal{P}$ and $\mathcal{Q}$. 

\begin{figure}[t]
	\centering {\centering\includegraphics[width=0.8\columnwidth]{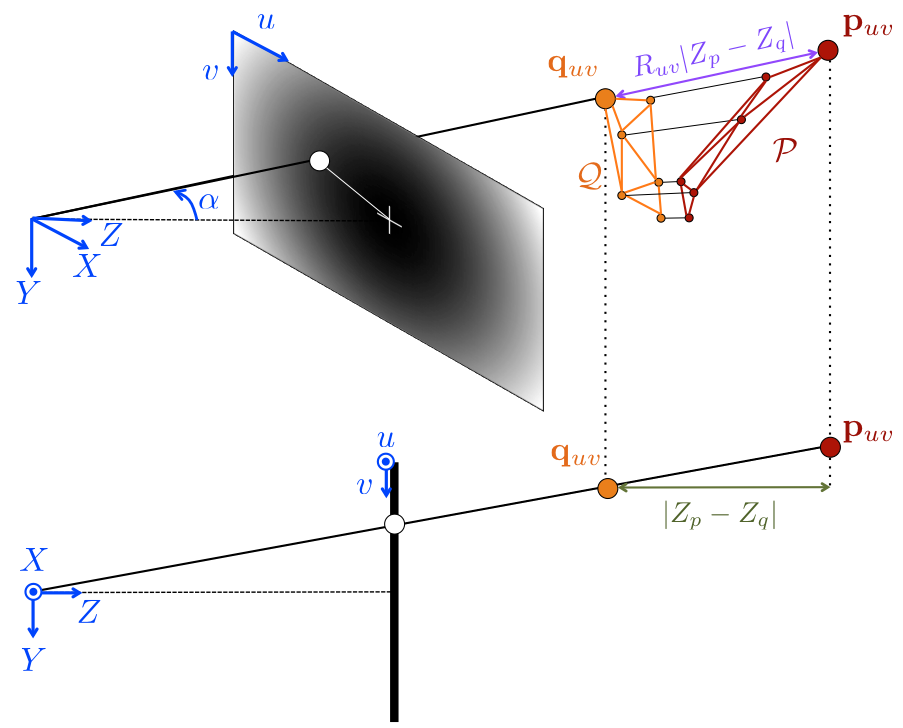}}
	\caption{Geometric interpretation of $R_{uv}$. Perspective (top) and lateral (bottom) view of the pinhole projection of point clouds ${\mathcal{Q}}$ (orange) and ${\mathcal{P}}$ (red). We show (for the white pixel) that the depth offset $\lvert Z_p - Z_q \rvert$ (green) differs from the point distance (purple) by scalar factor $R_{uv}$. This factor is constant on ellipses centred at the camera principal point, and increases with the sizes of the ellipses. In the top figure, we outline these elipses, which characterize image $\mathbf{R}$. We also show (in blue) the camera and image frames, as well as $\alpha$, the angle between the pixel projecting ray and the optical axis.}
	\label{Fig:Coeff}
    \vspace{-0.4cm}
\end{figure}

A common way of defining $d$ is to set it equal to the mean of the distances between all the pairs $\left(\mathcal{P},\mathcal{Q}\right)$ of points matched in the two point clouds:
\begin{equation}
d \left( \mathcal{P}, \mathcal{Q} \right) = \frac{1}{\dim{\mathcal{M}}} \sum_{\left(\mathbf{p},\mathbf{q}\right) \in {\mathcal{M}}}
\norm{\mathbf{p} - \mathbf{q}}_2.
\label{eq:meanDist}
\end{equation}
Alternative approaches are the Chamfer distance or the Earth Mover's Distance~\cite{FanSG16}. However, all these metrics require computing ${\mathcal{M}}$, the set denoting the association between the point clouds, which is defined as:
\begin{equation}
{\mathcal{M}} = \left\{ \left(\mathbf{p}_l, \mathbf{q}_l \right) : \forall \mathbf{p}_l \in \mathcal{P}, \mathbf{q}_l = \underset{m}{\mathrm{argmin}} \norm{\mathbf{p}_l - \mathbf{q}_m}_2 \in \mathcal{Q} 
\right\}.
\end{equation}
This matching operation can be solved by the well known point-to-point ICP algorithm~\cite{icp:94}. Although it is one of the widely used algorithms for aligning three dimensional models given an initial guess of the required rigid transformation, ICP's complexity (quadratic with the point cloud dimension) makes it very slow. 

In this pilot study, the task's complexity is reduced, since the robot can realize a limited number of actions. For human-like molding -- which requires many more action types and positions -- the distance computation time becomes crucial, at both learning and planning steps. In this perspective, a metric based on depth images will scale up much better than one based on point clouds. Furthermore, a point cloud uses approximately \%60 more memory than the equivalent depth image. Hence, using point clouds will substantially increase memory usage and may not be suitable for deep learning on consumer GPUs.

For all these reasons, we propose to match points from $\mathcal{P}$ and $\mathcal{Q}$ which project to the same pixel in images $\mathbf{L}_P$ and $\mathbf{L}_Q$ (see Fig.~\ref{Fig:Coeff}).
Naming
\begin{equation}
\left\{
\begin{array}{l}
\mathbf{p}_{uv} = \mathbf{m} \left( u, v, L^p_{uv}\right) \\
\mathbf{q}_{uv} = \mathbf{m} \left( u, v, L^q_{uv}\right)
\end{array}
\right.
\end{equation}
with $L^p_{uv}$ (respectively, $L^q_{uv}$) the luminance of pixel $\left(u,v\right)$ in depth image $\mathbf{L}_{P}$ ($\mathbf{L}_Q$), we propose to use:
\begin{equation}
{\mathcal{M}} = \left\{ 
\mathbf{p}_{uv}, 
\mathbf{q}_{uv}, \forall \left(u,v\right) \in \mathcal{I} \right\}.
\end{equation}

Note that since their pixel projection $\left(u,v\right)$ is the same, $\mathbf{p}_{uv}$ and $\mathbf{q}_{uv}$ have identical $x$ and $y$. Applying first~(\ref{eq:pinhole2}) and then~(\ref{eq:depth}), we can express their relative distance as:
\begin{equation}
\begin{split}
\norm{\mathbf{p} - \mathbf{q}}_2 &= \sqrt{\left(x Z_p- x Z_q\right)^2+\left(y Z_p- y Z_q\right)^2+\left(Z_p-Z_q\right)^2}\\
&= \sqrt{1 + x^2 + y^2} \; \lvert Z_p- Z_q \rvert \\
&= R_{uv} \frac{Z_{max}-Z_{min}}{2^b-1} \lvert L^p_{uv} - L^q_{uv} \rvert.
\end{split}
\label{eq:distPix}
\end{equation}
In this equation we have noted: $Z_{p} = Z \left( L^{p}_{uv} \right)$, $Z_{q} = Z \left( L^{q}_{uv} \right)$ and $R_{uv} = \sqrt{1 + \left(\frac{u-u_0}{f_x} \right)^2 +  \left( \frac{v-v_0}{f_y} \right) ^2 }$. The value $R_{uv}$ is independent from the point clouds, is constant on ellipses centered at the camera principal point (where it is $1$), and it increases with the sizes of the ellipses. A geometric interpretation when $f_x = f_y$ is $R_{uv} = 1 / cos \alpha$, with $\alpha$ the angle between the pixel projecting ray and the optical axis (see Fig.~\ref{Fig:Coeff}). For the Intel RealSense D435, with nominal parameters\footnote{See \url{https://www.intelrealsense.com/depth-camera-d435}} $u_0=424$, $v_0=240$, $f_x=415$, $f_y=373$, $R_{uv}$ varies between $\underline{R}=1$ (for the principal point) and $\overline{R}=1.568$ (for the four corner pixels).

Injecting~(\ref{eq:distPix}) in~(\ref{eq:meanDist}) for all the pixels yields:
\begin{equation}
d \left( \mathcal{P}, \mathcal{Q} \right) =  \frac{Z_{max}-Z_{min}}{wh\left(2^b-1\right)} 
\sum_{\left(u,v\right) \in \mathcal{I}}
R_{uv} \lvert L^p_{uv} - L^q_{uv} \rvert.
\end{equation} 

Since $Z_{max}$ and $Z_{min}$ are the same for all pairs of images $\left( \mathbf{L}_P, \mathbf{L}_Q \right)$, we can remove them from the equation. We can also divide by $\overline{R}$, to redefine the distance on unit interval $\mathcal{U}$:
\begin{equation}
d \left( \mathcal{P}, \mathcal{Q} \right) = \frac{
	\sum_{\left(u,v\right)\in \mathcal{I}}
	R_{uv} \lvert L^p_{uv} - L^q_{uv} \rvert}{wh\left(2^b-1\right)\overline{R}} \in \mathcal{U} = \left[0, 1 \right].
\label{eq:dist}
\end{equation}
Let us now define matrix
\begin{equation}
	\mathbf{R} = \frac{1}{\left(2^b-1\right)\overline{R}} \left(
	\begin{array}{ccc}
		R_{11} & \cdots & R_{1w} \\
		\vdots  & \ddots & \vdots \\
		R_{h1} & \cdots & R_{hw}
	\end{array}
	\right)\label{eq:radMatrix}
\end{equation} 
shown in Fig.~\ref{Fig:Coeff} for the Intel RealSense D435. Then, naming
\begin{equation}
\mathbf{I} = \mathbf{R} \circ \mathbf{L} \in \mathcal{U}^{w\times h},
\label{eq:newMat}
\end{equation}
with $\circ$ the Hadamard product, $d$ can be expressed as a function of these new normalized images $\mathbf{I}$:
\begin{equation}
d \left( \mathcal{P}, \mathcal{Q} \right) = d \left( \mathbf{I}_p , \mathbf{I}_q \right) =
\frac{1}{wh} \sum_{\left(u,v\right)\in \mathcal{I}} \lvert I^{p}_{uv} - I^{q}_{uv} \rvert.
\label{eq:distNewImg}
\end{equation}
Therefore, in the rest of the paper (except when specified) we work with images $\mathbf{I}$, rather than $\mathbf{L}$. 

We have analytically proved that our distance is equivalent to the weighted pixelwise difference between the two depth images. To illustrate the difference between our distance and an unweighted version, refer to the example of Fig.~\ref{counterexample}. The same two point clouds, red and orange,  each composed of three points, are seen with different camera orientations in Fig.~\ref{counterexample}(left) and Fig.~\ref{counterexample}(right). With our technique, the distance between the point clouds is the same in Fig.~\ref{counterexample}(left) and Fig.~\ref{counterexample}(right), as it should be (i.e., viewpoint invariant). Instead, with the unweighted version, which merely uses the depth differences, the distance in Fig.~\ref{counterexample}(right) will be much smaller than in Fig.~\ref{counterexample}(left). To summarize, our method calculates the geometric distance between two points instead of its projection along the camera axis (i.e., instead of the difference between the two points' depths). Hence, we think that our metric is faster and at least as easy to interpret than other similar metrics, such as Chamfer distance or the Earth Mover's Distance. Thus, it can also be useful for planning algorithms.

\begin{figure}[t]
	\centering 
	\includegraphics[trim=0 4cm 0 0, clip,width=0.8\columnwidth]{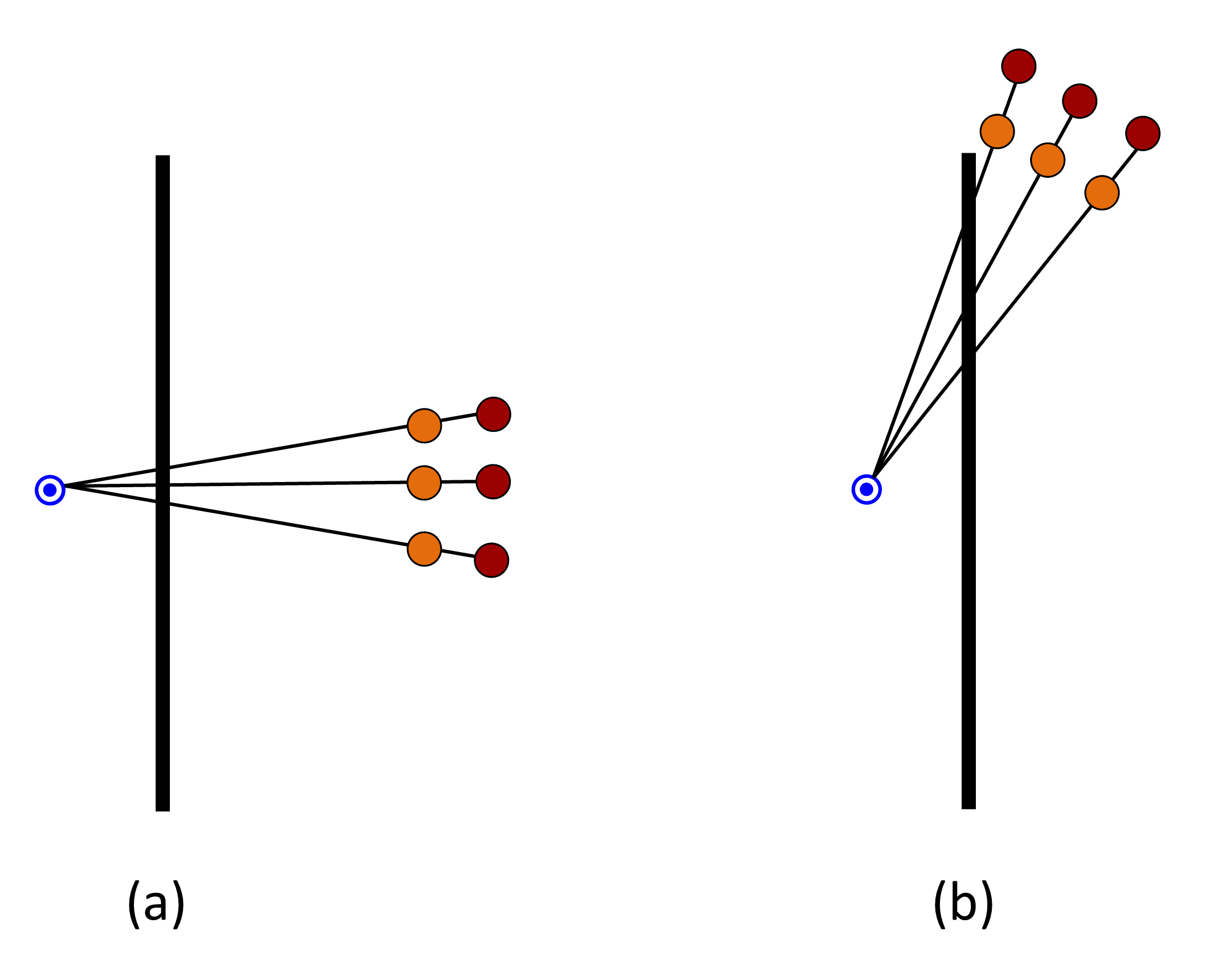}
	\caption{The same two point clouds, red and orange, each composed of three points, seen with different camera orientations (left) and (right).}
	\label{counterexample}
    \vspace{-0.6cm}
\end{figure}

\begin{prop}
If both $\mathcal{P}$ and $\mathcal{Q}$ are visible point clouds, $d \left( \mathcal{P}, \mathcal{Q} \right) = 0$ if and only if $\mathcal{P} = \mathcal{Q}$. Therefore, $d$ is an appropriate metric for the \textit{Point Cloud Shaping Task}.
\end{prop}

\begin{proof} 
Since $d \left( \mathcal{P}, \mathcal{Q} \right)$ is a sum of absolute values, if it is null, $I^{p}_{uv} = I^{q}_{uv} \forall \left(u,v\right)\in \mathcal{I}$. Since each $I^{p}_{uv}$ (respectively, $I^{q}_{uv}$) is proportional to $L^p_{uv}$ (respectively, $L^q_{uv}$) by the corresponding element of $\mathbf{R}$, if $d \left( \mathcal{P}, \mathcal{Q} \right) = 0$ then $L^p_{uv} = L^q_{uv}$. Then, mapping the two images with $\mathbf{m}$ will yield identical point clouds, so if $d \left( \mathcal{P}, \mathcal{Q} \right) = 0$ then $\mathcal{P} = \mathcal{Q}$.
Conversely: if the point clouds are identical ($\mathcal{P} = \mathcal{Q}$), applying $\mathbf{m}^{-1}$ to each one will yield identical depth images ($\mathbf{L}_p = \mathbf{L}_q$). Then, applying~(\ref{eq:newMat}) followed by~(\ref{eq:distNewImg}) leads to $d \left( \mathcal{P}, \mathcal{Q} \right) = 0$. 
\end{proof}

\section{Predicting the action effect on a point cloud}
\label{sec:predPCloud}

In this section, we explain how to predict the effect of an action on a point cloud.

\subsection{Outline}

At each iteration, action $\mathbf{a}$ changes the current point cloud $\mathcal{P}$, to produce a new point cloud $\mathcal{P}_{\mathbf{a}}$, according to~(\ref{eq:markovModel}). Since $\mathbf{a}$ can be any action within ${\mathcal{A}} = \left\{ \mathbf{a}^{ij} \right\}$, we must rewrite~(\ref{eq:markovModel}) for each possible $\mathbf{a}^{ij}$. We can do this by embedding the effect of $\mathbf{a}^{ij}$ within the expression of $\mathbf{f}$:
\begin{equation}
{\mathcal{P}}_{\mathbf{a}^{ij}} = \mathbf{f}^{ij}\left({\mathcal{P}}\right),
\label{eq:allActions}
\end{equation}
with:
\begin{equation}
\begin{array}{ll}
\mathbf{f}^{ij}: & \mathcal{R}^{3\times w \times h} \rightarrow \mathcal{R}^{3\times w \times h}. 
\end{array}
\label{eq:pCloudChange2}
\end{equation}
To predict the effect of \textit{each} possible action $\mathbf{a}^{ij}$, we need to know each $\mathbf{f}^{ij}$. Both the domain and codomain of these functions have very high dimension ($3 \times w \times h$), and we need $\dim {\mathcal{A}}$ of such functions! This clearly makes the prediction problem intractable, unless we exploit the other work assumptions to simplify it. 

\subsection{Simplifications}

First, since both the prior and posterior point clouds are visible, we can project them on the corresponding images, by applying~(\ref{eq:PtoI}) followed by~(\ref{eq:newMat}). Then, we can rewrite~(\ref{eq:pCloudChange2}) as:
\begin{equation}
\begin{array}{ll}
\mathbf{f}^{ij}: & \mathcal{U}^{w\times h} \rightarrow \mathcal{U}^{w\times h} \\
& \mathbf{I}  \mapsto
\mathbf{I}_{\mathbf{a}^{ij}}.
\end{array}
\label{eq:imgChange}
\end{equation}
By operating on images instead of point clouds, we have reduced the dimensions of the functions' domain and codomain from $3 \times w \times h$ to $w \times h$.

We can further simplify the problem, by considering \underline{Hypothesis 2: the effect of actions is bounded}. The projection of box $\mathcal{B}_{ij}$ in the depth image can be embedded within a rectangular region of interest (ROI). Table~\ref{table:RoiCorners} gives the coordinates of the top left $\left(\underline{u}, \underline{v} \right)$ and bottom right $\left(\overline{u},\overline{v}\right)$ pixel of this ROI, depending on the signs of $X^\pm$ and $Y^\pm$ ($Z^\pm > 0$ for all points of a visible point cloud). We derived the pixel coordinates in the second column of the table by applying~(\ref{eq:PtoI}), rounding to the closest integer (with $\nint{}$) and then clamping the result in the field of view. To illustrate this result, we show in Fig.~\ref{fig:roiBox} two examples: a red bounding box with all four values ($X^+$, $X^-$, $Y^+$, $Y^-$) negative and a green bounding box where all four values are positive.

\begin{table}[t]
	\caption{Pixel coordinates of the top left $\left(\underline{u}, \underline{v} \right)$ and bottom right $\left(\overline{u},\overline{v}\right)$ pixel of the rectangle embedding the image of $\mathcal{B}$.}
	\centering
	\begin{tabular}{|c|c|}
	\cline{1-2}
	\textbf{$X^-$} & $\underline{u}$ \\
	\hline 
	$\geq 0$ & $\mathrm{max} \left( 1, \mathrm{min} \left( \nint{f_xX^- / Z^+ + u_0}, w \right) \right)$\\
	$<0$ & $\mathrm{max} \left( 1, \mathrm{min} \left( \nint{f_xX^- / Z^- + u_0}, w \right) \right)$\\
	\hline
	\textbf{$Y^-$} & $\underline{v}$ \\
	\hline 
	$\geq 0$ & $\mathrm{max} \left( 1, \mathrm{min} \left( \nint{f_yY^- / Z^+ + v_0}, h \right) \right)$\\
	$<0$ & $\mathrm{max} \left( 1, \mathrm{min} \left( \nint{f_yY^- / Z^- + v_0}, h \right) \right)$\\
	\hline
	\textbf{$X^+$} & $\overline{u}$ \\
	\hline 
	$\geq 0$ & $\mathrm{max} \left( 1, \mathrm{min} \left( \nint{f_xX^+ / Z^- + u_0}, w \right) \right)$\\
	$<0$ & $\mathrm{max} \left( 1, \mathrm{min} \left( \nint{f_xX^+ / Z^+ + u_0}, w \right) \right)$\\
	\hline
	\textbf{$Y^+$} & $\overline{v}$ \\
	\hline 
	$\geq 0$ & $\mathrm{max} \left( 1, \mathrm{min} \left( \nint{f_yY^+ / Z^- + v_0}, h \right) \right)$\\
	$<0$ & $\mathrm{max} \left( 1, \mathrm{min} \left( \nint{f_yY^+ / Z^+ + v_0}, h \right) \right)$\\
	\hline
	\end{tabular}
	\label{table:RoiCorners}
     \vspace{-0.5cm}

\end{table}

Note that the clamping can be removed, since under \underline{Hypothesis 2}, all $\mathcal{B}_{ij}$ are entirely visible. This hypothesis (equivalent to $\underline{u}, \overline{u} \in \left[1, w\right]$ and $\underline{v}, \overline{v} \in \left[1, h\right]$) can be guaranteed by setting the action positions $\mathbf{t}$ such that:
\begin{equation}
\left\{
\begin{array}{l}
\frac{\left(1-u_0\right)Z^+}{f_x} + \Delta X \leq X \leq \frac{\left(w-u_0\right)Z^-}{f_x} - \Delta X \\
\frac{\left(1-v_0\right)Z^+}{f_y} + \Delta Y \leq Y \leq \frac{\left(h-v_0\right)Z^-}{f_y} - \Delta Y 
\end{array}
\right.
.
\end{equation}

Under \underline{Hypothesis 3} (i.e., $Z^- \approx Z^+ \approx Z$) the ROI top left and bottom right pixels are always:
\begin{equation}
\left\{
\begin{array}{l}
\left(\underline{u}, \underline{v} \right) = \left( f_x X^- / Z + u_0, f_y Y^- / Z + v_0 \right) \\
\left(\overline{u},\overline{v}\right) = \left( f_x X^+ / Z + u_0, f_y Y^+ / Z + v_0 \right)
\end{array}
\right.
.
\end{equation} 
%Since all the points modified by action $\mathbf{a}_{ij}$ are within $\mathcal{B}_{ij}$, and $\mathcal{B}_{ij}$ is visible within the ROI, the corresponding pixels cannot be outside the ROI.
Note that since all actions have the same $Z$, the ROI width and height, $\underline{w} = \overline{u} - \underline{u} = 2 f_x \Delta X / Z$ and $\underline{h} = \overline{v} - \underline{v} = 2 f_y \Delta Y / Z$, are unique for a given $\mathbf{a}^{i}$ (i.e., they do not depend on $\mathbf{t}$).

\begin{figure}[t]
	\centering 
	\includegraphics[width=0.8\columnwidth]{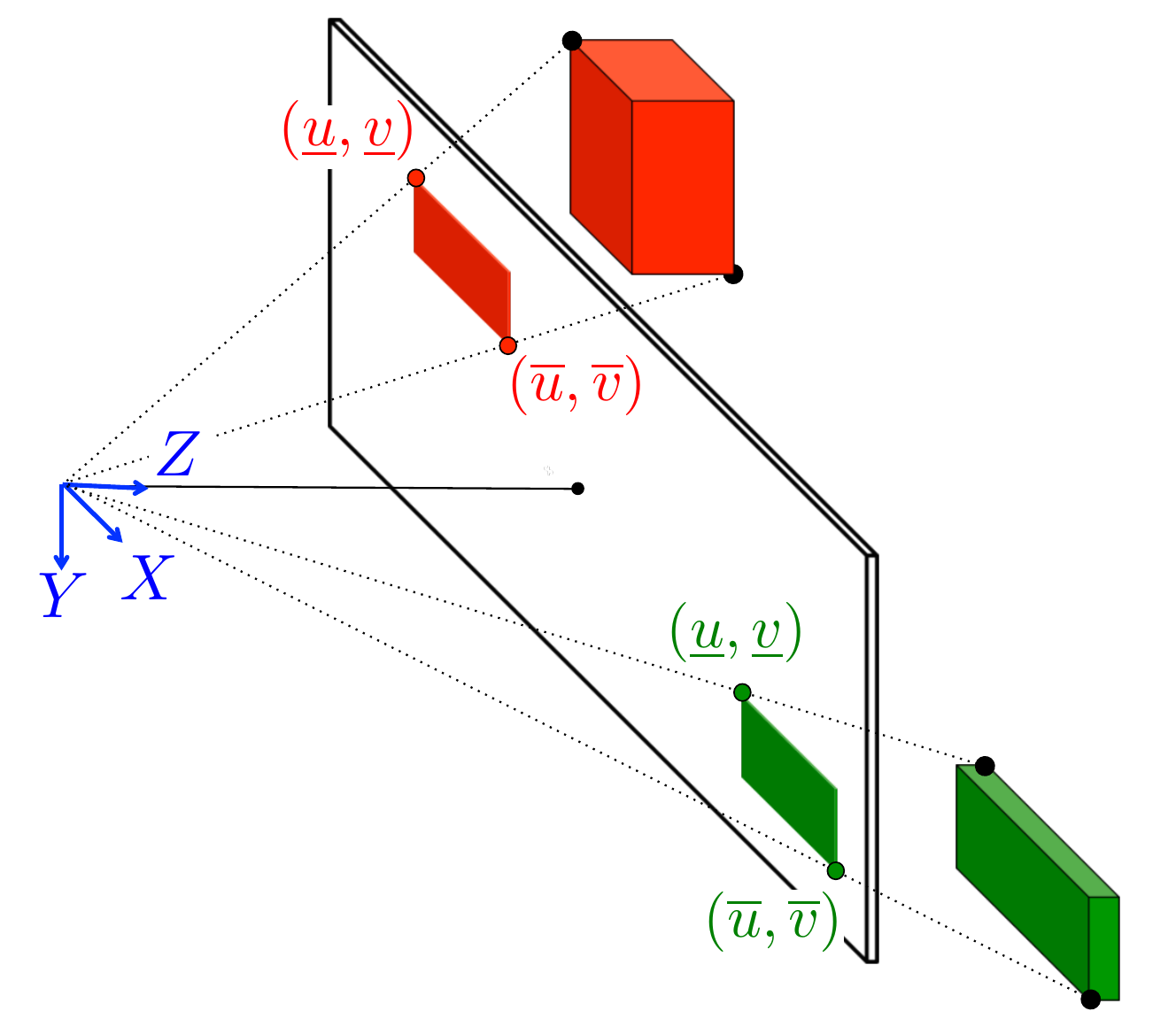}
	\caption{Two examples of bounding boxes (a red one with all four values $X^+$, $X^-$, $Y^+$, $Y^-$ negative and a green one, where all four values are positive) with the corresponding rectangular regions of interest (ROI) in the image plane. The ROI are characterized by the coordinates of the top left $\left(\underline{u}, \underline{v} \right)$ and bottom right $\left(\overline{u},\overline{v}\right)$ pixel. Note that these pixels are the projections of different corners of the two bounding boxes.}
	\label{fig:roiBox}
    \vspace{-0.4cm}
\end{figure}

Let us denote
\begin{equation}
\underline{\mathbf{I}} = \left(
\begin{array}{ccc}
I_{\underline{u}\hspace{0.3mm}\underline{v}} & \cdots & I_{\overline{u}\underline{v}} \\
\vdots  & \ddots & \vdots \\
I_{\underline{u}\overline{v}} & \cdots & I_{\overline{u}\hspace{0.3mm}\overline{v}} \\
\end{array}
\right) \in \mathcal{U}^{\underline{w}\times \underline{h}}
\label{eq:predStep1}
\end{equation}
the submatrix of ${\mathbf{I}}^{}$, containing the values of the pixels which could change after an action in $\mathcal{B}_{ij}$. Action $\mathbf{a}^{ij}$ will transform $\underline{\mathbf{I}}$ into a new image:
\begin{equation}
\underline{\mathbf{I}}_{\mathbf{a}^{ij}} = \left(
\begin{array}{ccc}
I^{ij}_{\underline{u}\hspace{0.3mm}\underline{v}} & \cdots & I^{ij}_{\overline{u}\underline{v}} \\
\vdots  & \ddots & \vdots \\
I^{ij}_{\underline{u}\overline{v}} & \cdots & I^{ij}_{\overline{u}\hspace{0.3mm}\overline{v}} \\
\end{array}
\right) \in \mathcal{U}^{\underline{w}\times \underline{h}}.
\label{eq:predStep2}
\end{equation}
By focusing on a region of the image, we have further reduced the dimensions of the domain and codomain, from $w \times h$ to $\underline{w}  \times \underline{h}$:
\begin{equation}
\begin{array}{ll}
\mathbf{f}^{ij}: & \mathcal{U}^{\underline{w}\times \underline{h}} \rightarrow \mathcal{U}^{\underline{w}\times \underline{h}} \\
& \underline{\mathbf{I}}  \mapsto
\underline{\mathbf{I}}_{\mathbf{a}^{ij}}.
\end{array}
\label{eq:imgRoiChange}
\end{equation}
Yet, we still need to know the function $\mathbf{f}^{ij}$ for each combination of action $i$ and position $j$.

To further simplify the problem, we exploit \underline{Hypothesis 4: actions are position invariant}. This last simplification reduces the number of functions required: instead of $\dim {\mathcal{A}}$ functions, it is sufficient to know the function of each $\mathbf{a}^i$:
\begin{equation}
\begin{array}{ll}
\mathbf{f}^{i}: & \mathcal{U}^{\underline{w}\times \underline{h}} \rightarrow \mathcal{U}^{\underline{w}\times \underline{h}} \\
& \underline{\mathbf{I}}  \mapsto
\underline{\mathbf{I}}_{\mathbf{a}^{i}}.
\end{array}
\label{eq:lastMapping}
\end{equation}
In the next subsection, we explain how we determine the function $\mathbf{f}^{i}$ predicting the effect of action $\mathbf{a}^i$ on a region of interest $\underline{\mathbf{I}}$.

\subsection{Predicting the action effect on a region of interest}

To predict the effect of $\mathbf{a}^i$ (noted $\mathbf{a}$ for simplicity) on $\underline{\mathbf{I}}$, we have tested three different functions $\mathbf{f}$. To design them, we exploit -- for each $\mathbf{a}$ -- a prerecorded dataset of pairs of regions of interest prior/posterior to the action:
\begin{equation}
{\mathcal{G}} = \left\{ \left( \underline{\mathbf{I}}, \underline{\mathbf{I}}_{\mathbf{a}} \right)_1, \dots , \left( \underline{\mathbf{I}}, \underline{\mathbf{I}}_{\mathbf{a}} \right)_{\dim \mathcal{G}}
\right\}
\label{eq:datasetPairs}
\end{equation}
The three designs are: \textit{difference compensation}, \textit{refined generator}, and \textit{difference compensation with refined generator} (denoted respectively $\mathbf{f}_\mathrm{D}$, $\mathbf{f}_\mathrm{CR}$, and $\mathbf{f}_\mathrm{DCR}$). The designs are outlined in Fig.~\ref{fig:pred}; we will detail them in this section, after having explained how to build the dataset~(\ref{eq:datasetPairs}) of prior/posterior pairs. 

\subsubsection{Building a dataset of prior/posterior image ROI}

We record a sequence of depth images $\mathbf{L}$, while repeatedly applying action $\mathbf{a}$. We use~(\ref{eq:newMat}) to convert each image $\mathbf{L}$ into the corresponding $\mathbf{I}$.

We remove all images where the environment is occluded (typically, during the action), and extract pairs of prior/posterior images, which represent the image before/after the action. Then, we manually annotate, in these images, the ROI which contains a ``visible'' (by human standards) difference in depth\footnote{The annotation could be automated using standard image processing techniques, similar to the ones we developed in~\cite{cherubini:ijrr:2020}.}. Afterwards, we adjust the width $\underline{w}$ and height $\underline{h}$ of all the ROI to the mean of all the annotated ones, and we fix their positions so that they are identical in both images of the pair. These ROI will constitute~(\ref{eq:datasetPairs}).

%\textcolor{blue}{The cropped images will be used as ground truth ROI anterior-posterior pairs and denoted $\mathbf{^+\underline{I}}$ for the anterior and $\mathbf{^+\underline{I}^a}$ for the posterior, according to~(\ref{eq:predStep1}) and~(\ref{eq:predStep2}). In following subsections, we present the work done on a pair of images, $n$ and we remove the subscripts $n$ for the simplicity.}

\subsubsection{Difference compensation}

A first simple image processing operation consists in approximating the effect of $\mathbf{a}$ on $\underline{\mathbf{I}}$, by the average (over the dataset~(\ref{eq:datasetPairs})) of the ROI differences
\begin{equation}
\Delta \underline{\mathbf{I}} = 
\frac{1}{\dim \mathcal{G}} \sum^{\mathcal{G}} \left( \underline{\mathbf{I}}_{\mathbf{a}} - \underline{\mathbf{I}}\right). 
\label{eq:DiffImage}
\end{equation}
This image average is added to $\underline{\mathbf{I}}$, to obtain a rough prediction of the ROI, after having applied $\mathbf{a}$:
\begin{equation}
\begin{array}{ll}
\mathbf{f}_\mathrm{D}: & \mathcal{U}^{\underline{w}\times \underline{h}} \rightarrow \mathcal{U}^{\underline{w}\times \underline{h}} \\
& \underline{\mathbf{I}}  \mapsto
\underline{\mathbf{I}} + \Delta \underline{\mathbf{I}}.
\end{array}
\label{eq:Doutput}
\end{equation}
Although this operation may lead some pixels to saturate (to $0$ or $1$), this occurs very rarely in practice. 

\begin{figure*}[t]
	\centering 
	\includegraphics[width=1.8\columnwidth]{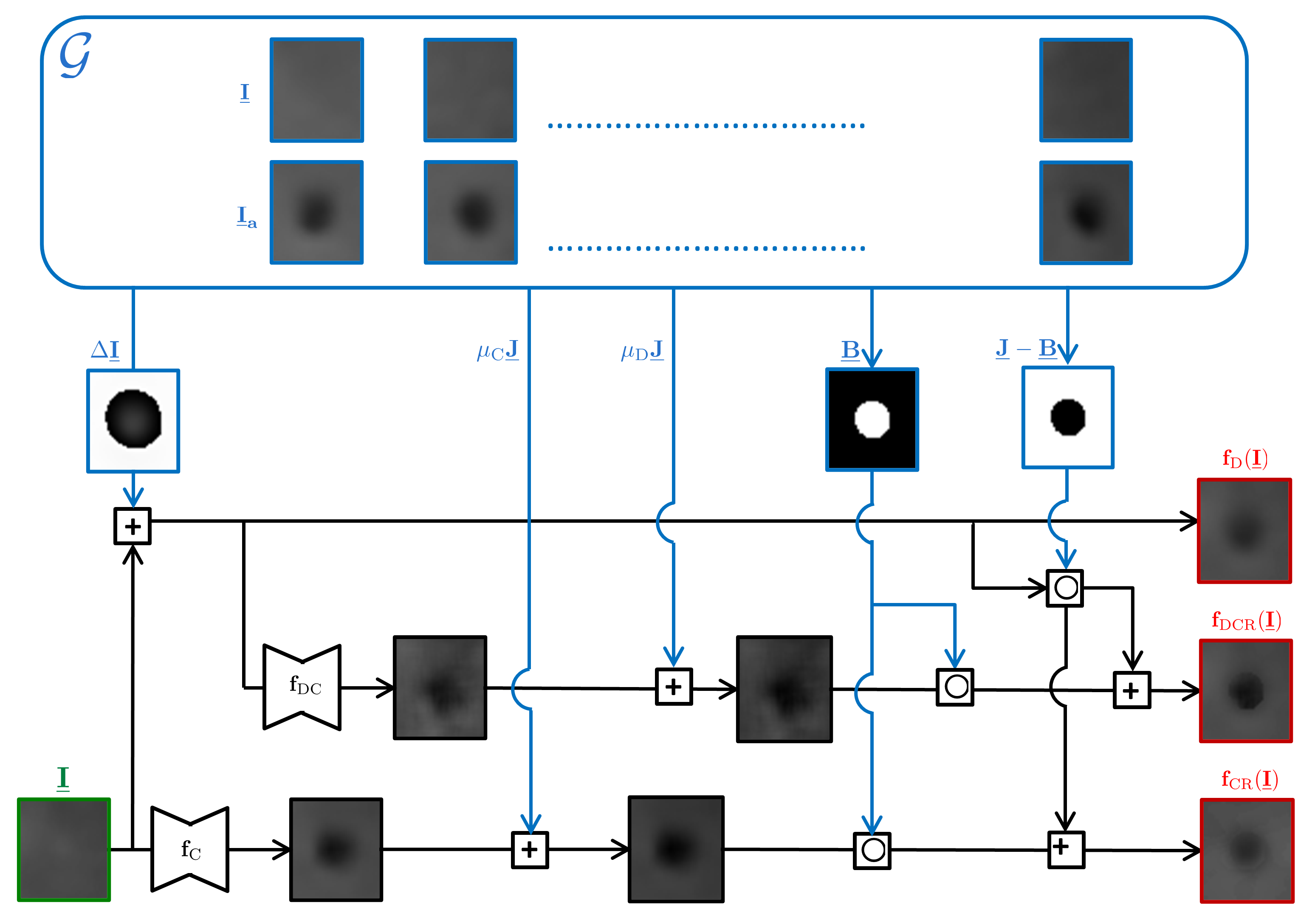}
	\caption{The three prediction functions $\mathbf{f}_\mathrm{D}$, $\mathbf{f}_\mathrm{DCR}$, and $\mathbf{f}_\mathrm{CR}$, illustrated (in red) for the poke action. For the purpose of illustration, in the image $\Delta \underline{\mathbf{I}}$ we have set all negative pixels to white (i.e., to unitary value).}
	\label{fig:pred}
    \vspace{-0.5cm}
\end{figure*}

\subsubsection{Refined Generator}

%\textcolor{red}{Indeed maybe we should call it Convolutional Autoencoder. Also, I think its role is more to add noise (normally~(\ref{eq:Doutput}) has blurred the image, by adding an average of images, which is blurred). Is this correct?} 

%\textcolor{blue}{I don't quite get why we should use the term adding (noise) rather than erasing. In that case, what is the noise that we add ? If we suppose that the blurrings are noise, then we can can call the $\underline{\mathbf{I}}_d$ noisy image and that we are looking for the denoised version using a CAE. If you look at~\cite{Jain2008NaturalID}, which is I believe the first paper that uses convolutional autoencoders for denoising, the blurred images are considered as noisy, and the outputs are blur free. We can also add that: " It was observed that using a small sample of training images, equal or better performance than state-of-the-art based on wavelets and Markov random fields can be achieved" to confirm our choice of CAE as described in the same paper.}

Our generator is a type of artificial neural network which is constituted by two main parts. The encoder maps the input into feature representations at multiple levels and the decoder maps the features onto the pixel space. We design and train (with the prior/posterior images from ${\mathcal{G}}$) a generator to generate a ROI from an input one: 
\begin{equation}
\begin{array}{ll}
\mathbf{f}_\mathrm{C}: & \mathcal{U}^{\underline{w}\times \underline{h}} \rightarrow \mathcal{U}^{\underline{w}\times \underline{h}} \\
& \underline{\mathbf{I}}  \mapsto
\mathbf{f}_\mathrm{C} \left( \underline{\mathbf{I}} \right).
\end{array}
\label{eq:Coutput}
\end{equation}
We have implemented $\mathbf{f}_\mathrm{C}$ as a modified U-net image generator~\cite{u-net}, similar to the one used by the pix2pix generative adversarial network~\cite{pix2pix2017}, with the following specifications  (illustrated in Fig.~\ref{fig:AEarch}). For the encoder, the size of the input layer varies according to the ROI dimensions $(\underline{w},\underline{h})$. The input is followed by 4 downscale layers, each consisting of two 3x3 convolution layers followed by a 2x2 maxpooling layer. Filter sizes of the convolution layers are 32, 64, 128 and 256, respectively. The decoder has 5 upscale layers and 1 output layer. Upscale layers consist of two 3x3 convolution layers, followed by a 2x2 upsampling layer. Filter size of convolution layers are 512, 256, 128, 64 and 32, respectively. Outputs of the second convolution layer from downscale layers are concatenated with the output of the symmetrical upsampling layer, before being fed to the following upscale layers. The output layer is a 1x1 convolution layer with $\tanh$ activation function. All the strides are set to 1 and rectified linear unit activation is used in all the convolution layers, except for the output layer. We use as loss function between the output of $\mathbf{f}_\mathrm{C}$ and the corresponding posterior $\underline{\mathbf{I}}_\mathbf{a}$ (ground truth or expected output, taken from ${\mathcal{G}}$), the distance $d$ defined in~(\ref{eq:distNewImg}):
\begin{equation}
	d \left( \underline{\mathbf{I}}_\mathbf{a} , \mathbf{f}_\mathrm{C} \left( \underline{\mathbf{I}} \right) \right) = \left|  \underline{\mathbf{I}}_\mathbf{a} - \mathbf{f}_\mathrm{C} \left( \underline{\mathbf{I}} \right) \right|.
	\label{eq:lossC}
\end{equation}
Note that this loss function resembles the mean absolute error, commonly used in Convolutional Neural Networks\footnote{See for instance the \textit{MeanAbsoluteError} class in \textit{Keras}: \url{https://keras.io/api/losses/regression_losses}}, with the addition of the $R_{uv} / \overline{R}$ weights. 

\begin{figure*}[t]
	\centering 
	\includegraphics[width=1.7\columnwidth]{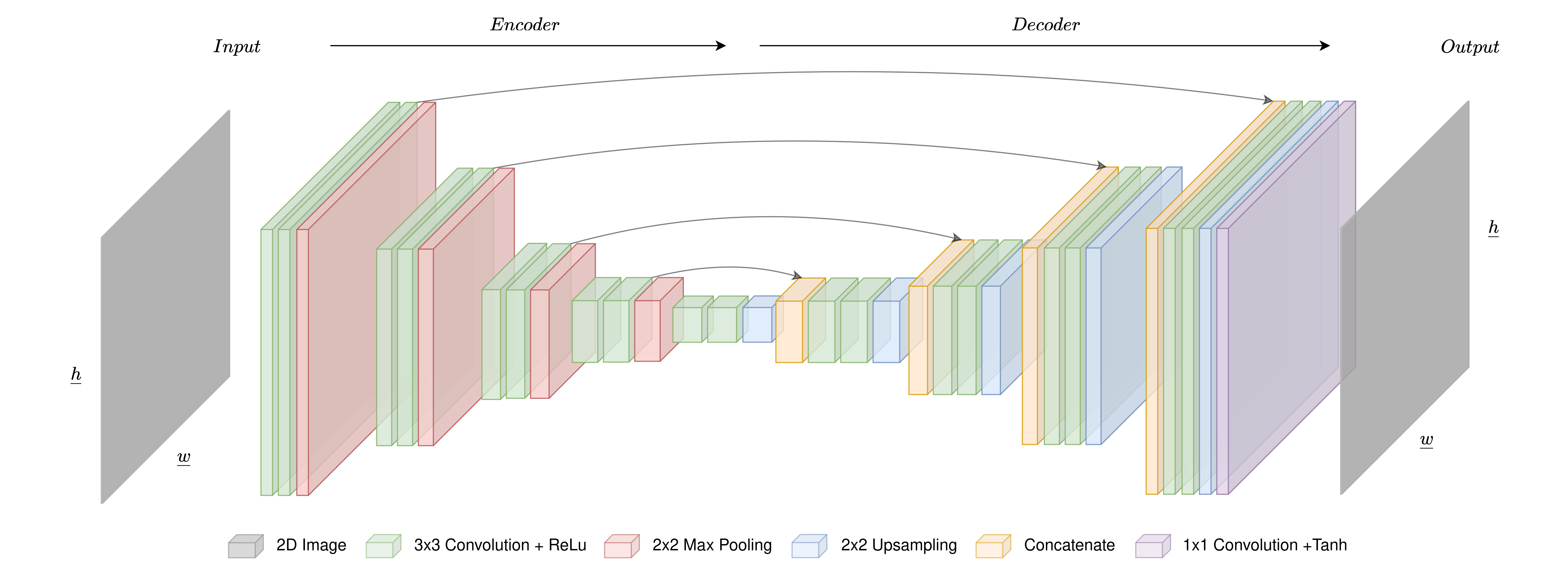}
	\caption{Architecture of the Generator Networks $\mathbf{f}_\mathrm{C}$ and $\mathbf{f}_\mathrm{DC}$.}
	\label{fig:AEarch}
    \vspace{-0.5cm}
\end{figure*}

We have realized experimentally that $\mathbf{f}_\mathrm{C}$ is not very accurate in predicting changes in the areas of $\underline{\mathbf{I}}$ the least affected by the action. In the areas the most affected by the action, the output of $\mathbf{f}_\mathrm{C}$ is accurate, apart from a slight difference in its mean pixel value. To solve both issues, we have added a \textit{refining step} (described below) to the output of the generator $\mathbf{f}_\mathrm{C}$. 

We start by building a boolean matrix $\underline{\mathbf{B}}$ by applying Otsu thresholding~\cite{otsu} followed by morphological erosion, to the dataset average image difference, $\Delta \underline{\mathbf{I}}$ defined in~(\ref{eq:DiffImage}). Then:
\begin{itemize}
    \item 
    pixels which are the least modified by the action are null in $\underline{\mathbf{B}}$; we set their values to those of the corresponding pixels in image $\underline{\mathbf{I}} + \Delta \underline{\mathbf{I}}$.
    \item
    pixels which are the most modified by the action are unitary in $\underline{\mathbf{B}}$; we shift their values, output by $\mathbf{f}_\mathrm{C}$, by an offset $\mu_\mathrm{C}$, such that their mean value is identical to that of the same pixels in $\underline{\mathbf{I}} + \Delta \underline{\mathbf{I}}$.
\end{itemize}
In summary:
\begin{equation}
\begin{array}{ll}
\mathbf{f}_\mathrm{CR}: & \mathcal{U}^{\underline{w}\times \underline{h}} \rightarrow \mathcal{U}^{\underline{w}\times \underline{h}} \\
& \underline{\mathbf{I}}  \mapsto
\underline{\mathbf{B}} \circ \left( \mathbf{f}_\mathrm{C} \left( \underline{\mathbf{I}}\right) + \mu_\mathrm{C} \underline{\mathbf{J}} \right) + \\
& \hspace{0.7cm} + \left( \underline{\mathbf{J}} - \underline{\mathbf{B}} \right) \circ \left( \underline{\mathbf{I}} + \Delta \underline{\mathbf{I}} \right),
\end{array}
\label{eq:RCoutput}
\end{equation}
with $\underline{\mathbf{J}}$ a ($\underline{w}$ x $\underline{h}$) matrix of ones, and $\mu_\mathrm{C} \in \left[ -1, 1\right]$ the mean of the non-zero pixels of $\underline{\mathbf{B}} \circ \left( \underline{\mathbf{I}} + \Delta \underline{\mathbf{I}} - \mathbf{f}_\mathrm{C} \left( \underline{\mathbf{I}}\right) \right)$.

%with $\neg$ denoting the logical \textit{not} operation.

%To improve the prediction $\underline{\mathbf{I}}_c$, we first isolate the effect of the $\mathbf{a}$ from the rest of the $\underline{\mathbf{I}}_c$. Then, we apply separate adjustments for the part which contains the action $^a\underline{\mathbf{I}}_c$ and the outside of the action area $^o\underline{\mathbf{I}}_c$. We approximate the effect of $\mathbf{a}$ as an image difference (posterior - prior), then average it over the dataset $\mathcal{G}$. The result can be used as a rough estimation for the effect of $\mathbf{a}$

%\begin{equation}
%\underline{\mathbf{I}}_d  = \frac{1}{\dim \mathcal{G}} \sum^{\mathcal{G}} \left( \underline{\mathbf{I}}_{\mathbf{a}} - \underline{\mathbf{I}}\right). 
%\label{eq:Doutput}
%\end{equation}

%From $\underline{\mathbf{I}}_d$, we create an image mask $\underline{\mathbf{I}}_m$ using adaptive thresholding with Otsu's method~\cite{otsu} followed by a morphological erosion filter. We denote $^a\underline{\mathbf{I}}$ as the output of the Hadamard product between $\underline{\mathbf{I}}$ and $\underline{\mathbf{I}}_m$

%\begin{equation}
%	\underline{\mathbf{I}}_m=
%	\begin{cases}
%		\mathbf{1} 
%		& \text{if $\mathbf{a}$}\\ 
%		\mathbf{0} 
%		& \text{otherwise}\\
%	\end{cases}  
%	\label{eq:mask}
%\end{equation}

\subsubsection{Difference compensation with Refined Generator}

The difference compensation function $\mathbf{f}_\mathrm{D}$ has the advantage of not requiring a training process. A drawback is that adding an image average introduces noise (blur) on $\underline{\mathbf{I}}$. The authors of~\cite{Jain2008NaturalID} have shown that to denoise images, convolutional autoencoders, which are simpler architectures than our generator, outperform  wavelets and Markov random fields on small samples of training images. Therefore, our third approach consists in passing $\underline{\mathbf{I}} + \Delta \underline{\mathbf{I}}$ to a new \textit{denoising} generator, denoted $\mathbf{f}_\mathrm{DC}$, and then refining the output, as in~(\ref{eq:RCoutput}). We have designed $\mathbf{f}_\mathrm{DC}$ with the same architecture as $\mathbf{f}_\mathrm{C}$. Yet, the decoder in $\mathbf{f}_\mathrm{DC}$ should map the latent space to a denoised version of its input. Therefore, we train this network with the prior images of ${\mathcal{G}}$ with $\Delta \underline{\mathbf{I}}$ added to them. The loss function of $\mathbf{f}_\mathrm{DC}$ is the distance between its output and the corresponding posterior ROI, $\underline{\mathbf{I}}_\mathbf{a}$ (from ${\mathcal{G}}$):
\begin{equation}
	d \left( \underline{\mathbf{I}}_\mathbf{a} , \mathbf{f}_\mathrm{DC} \left( \underline{\mathbf{I}} + \Delta \underline{\mathbf{I}} \right) \right) = \left|  \underline{\mathbf{I}}_\mathbf{a} - \mathbf{f}_\mathrm{DC} \left( \underline{\mathbf{Ia}} + \Delta \underline{\mathbf{I}} \right) \right|.
	\label{eq:lossD}
\end{equation}

The images output by $\mathbf{f}_\mathrm{DC}$ have the same issues seen for $\mathbf{f}_\mathrm{C}$. Therefore, we apply a similar refinement step with boolean matrix $\underline{\mathbf{B}}$, and design $\mathbf{f}_\mathrm{DCR}$ as:
\begin{equation}
\begin{array}{ll}
\mathbf{f}_\mathrm{DCR}: & \mathcal{U}^{\underline{w}\times \underline{h}} \rightarrow \mathcal{U}^{\underline{w}\times \underline{h}} \\
& \underline{\mathbf{I}}  \mapsto
\underline{\mathbf{B}} \circ \left( \mathbf{f}_\mathrm{DC} \left( \underline{\mathbf{I}} + \Delta \underline{\mathbf{I}}\right) + \mu_\mathrm{D} \underline{\mathbf{J}} \right) +\\
& \hspace{0.7cm} + \left( \underline{\mathbf{J}} - \underline{\mathbf{B}} \right) \circ \left( \underline{\mathbf{I}} + \Delta \underline{\mathbf{I}} \right),
\end{array}
\label{eq:DRCoutput}
\end{equation}
with $\mu_\mathrm{D} \in \left[ -1, 1\right]$ the mean of the non-zero pixels of $\underline{\mathbf{B}} \circ \left( \underline{\mathbf{I}} + \Delta \underline{\mathbf{I}} - \mathbf{f}_\mathrm{DC} \left( \underline{\mathbf{I}} + \Delta \underline{\mathbf{I}}\right) \right)$.

We have presented three alternative functions ${\mathbf{f}}$ for predicting the effect of an action on a region of interest $\underline{\mathbf{I}}$: $\mathbf{f}_\mathrm{D}$, $\mathbf{f}_\mathrm{CR}$, and $\mathbf{f}_\mathrm{DCR}$. Having chosen one of the three (either~(\ref{eq:Doutput}),~(\ref{eq:RCoutput}) or~(\ref{eq:DRCoutput})), we can replace the resulting (\textit{predicted}) submatrix in the original image $\mathbf{I}$, to obtain:
\begin{equation}
{\mathbf{I}}_{\mathbf{a}^{ij}} = \left(
\begin{array}{ccccccc}
I_{11} & \cdots & \cdots & \cdots & \cdots & \cdots & I_{1w} \\
\cdots & \cdots & I^{i}_{\underline{u}\hspace{0.3mm}\underline{v}} & \cdots & I^{i}_{\overline{u}\underline{v}} & \cdots & \cdots \\
\vdots  & \vdots  & \vdots  & \ddots & \vdots & \vdots  & \vdots\\
\cdots & \cdots & I^{i}_{\underline{u}\overline{v}} & \cdots & I^{i}_{\overline{u}\hspace{0.3mm}\overline{v}} & \cdots & \cdots \\
I_{h1} & \cdots & \cdots & \cdots & \cdots & \cdots & I_{hw} \\
\end{array}
\right).
\label{eq:predStep3}
\end{equation} 
Note that we have removed from the second term the $j$ dependency (pedix), which appeared in~(\ref{eq:predStep2}) and which is irrelevant since actions are position invariant.

% %Note that, some points may not be visible (e.g., due to occlusions). 
%Therefore, we can write~(\ref{eq:imgChange}) as:
%\begin{equation}
%{\mathbf{I}}_{\mathbf{a}^{ij}} = {\mathbf{I}} + \left[
%\begin{array}{ccccccc}
%0 & \cdots & \cdots & \cdots & \cdots & \cdots & 0 \\
%\cdots & \cdots & I^{ij}_{\underline{u}\hspace{0.3mm}\underline{v}} & \cdots & I^{ij}_{\overline{u}\underline{v}} & \cdots & \cdots \\
%\vdots  & \vdots  & \vdots  & \ddots & \vdots & \vdots  & \vdots\\
%\cdots & \cdots & I^{ij}_{\underline{u}\overline{v}} & \cdots & I^{ij}_{\overline{u}\hspace{0.3mm}\overline{v}} & \cdots & \cdots \\
%0 & \cdots & \cdots & \cdots & \cdots & \cdots & 0 \\
%\end{array}
%\right].
%\label{eq:imgChangeLin}
%\end{equation}

\section{Finding the best sequence of actions for point cloud shaping}\label{sect:treeSearch}

In this section, we explain how to compute the sequence of actions $\mathbf{s}_*= \left\{ \mathbf{a}_1, \dots, \mathbf{a}_K \right\}$ for \textit{point cloud shaping}, i.e., for modifying the environment from initial point cloud $\mathcal{P}_{0}$ to final desired point cloud $\mathcal{P}_*$. 

\begin{figure*}[t]
	\centering 
	\includegraphics[width=0.95\textwidth]{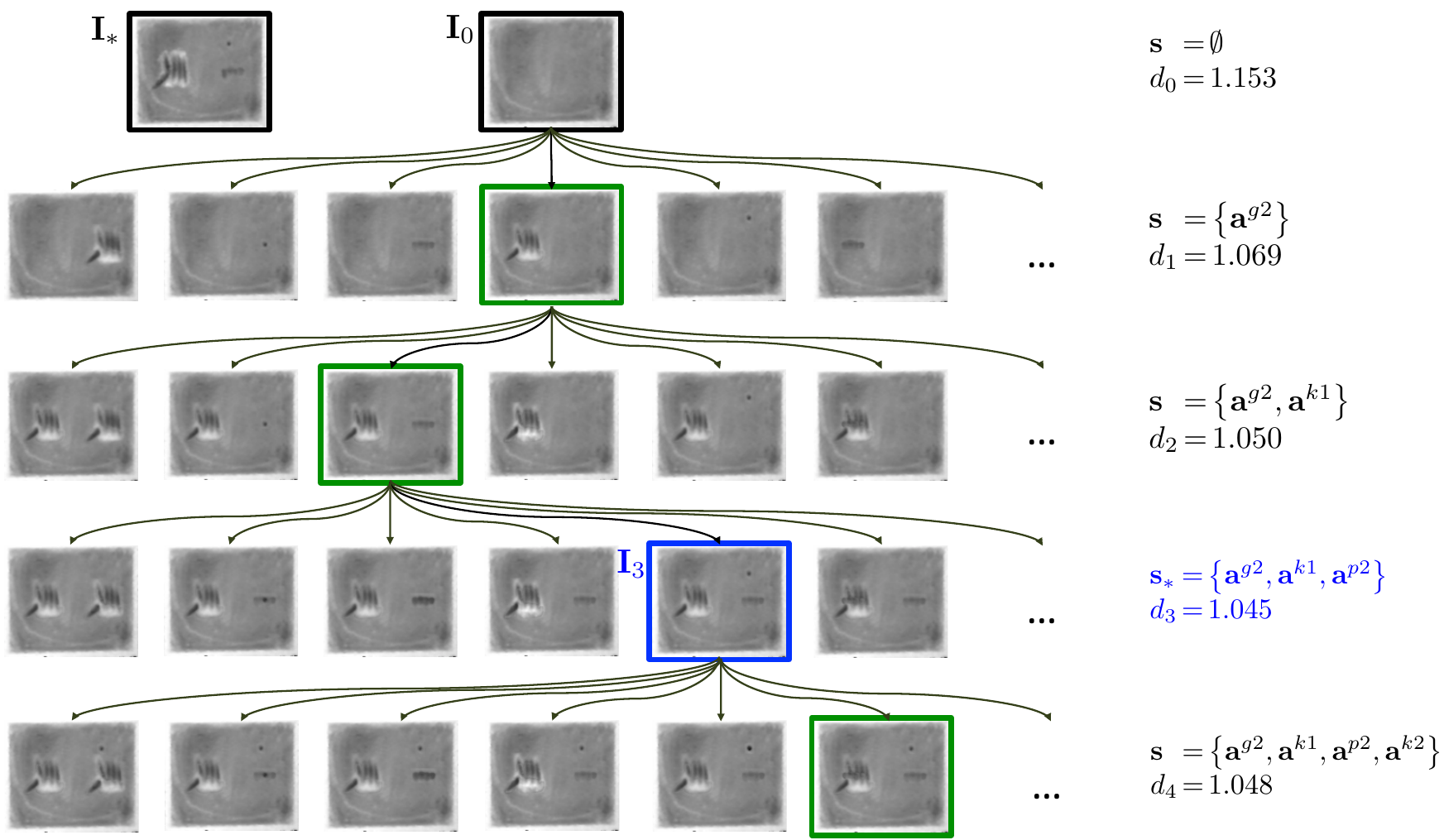}
	\caption{Forward search in the tree of posterior images, based on \textbf{Algorithm 1}. Given initial image $\mathbf{I}_0$ and desired image $\mathbf{I}_*$, at each step the sequence $\mathbf{s}$ is updated with the action which yields the image (green) that is the closest (has smallest distance $d$) to $\mathbf{I}_*$. In this example, the search stops after $3$ steps with the sequence $\mathbf{s}_*$ in blue, since any fourth action would increase the distance ($d_4 > d_3$). Note the similarity between the resulting image ($\mathbf{I}_3$ in blue) and $\mathbf{I}_*$. For the sake of clarity, at each step we only show $6$ of the total $\dim \mathcal{A}$ posterior images. Distances are expressed in mm.}
	\label{fig.treeImg}
    \vspace{-0.3cm}
\end{figure*}

By mapping point clouds to images with~(\ref{eq:PtoI}), we can instantiate this path planning problem in image space, i.e., find $\mathbf{s}$ which modifies the environment from $\mathbf{L}_0$ to $\mathbf{L}_*$ (i.e., from the depth image of $\mathcal{P}_0$ to that of $\mathcal{P}_*$). Furthermore, to express distance $d$ according to~(\ref{eq:distNewImg}), we will work with weighed normalized images $\mathbf{I}_0$ and $\mathbf{I}_*$ instead of $\mathbf{L}_0$ and $\mathbf{L}_*$ by applying~(\ref{eq:newMat}). Within a set of candidate sequences of actions ${\mathcal{S}}$ each leading from $\mathbf{I}_{0}$ to a different image $\mathbf{I}_{|\mathbf{s}}$, we will choose the one leading to the image that is the closest to $\mathbf{I}_*$:

\begin{equation}
	\mathbf{s}_* = \underset{\mathbf{s} \in {\mathcal{S}}}{\mathrm{argmin}} \;\; d \left( \mathbf{I}_{|\mathbf{s}}, \mathbf{I}_* \right).
\end{equation}

This sequence corresponds to the shortest path from $\mathbf{I}_0$ to $\mathbf{I}_*$, in the tree (shown in Fig.~\ref{fig.treeImg} with an example of forward search) of all possible images generated by all possible actions in ${\mathcal{A}}$. Finding the shortest path in a tree is a classic motion planning problem, which can be solved using many state of art algorithms (A$^*$, RRT, etc) out of scope here. We leave to path planning experts the pleasure of proposing the best among such algorithms. Since the focus of our paper is \textit{on how to build the tree of posterior images}, we propose the following forward search algorithm (refer to \textbf{Algorithm 1}):
\begin{enumerate}
	\item 
	at each step $k$, generate a tree from $\mathbf{I}_k$ to all possible (i.e., obtained after executing any action in ${\mathcal{A}}$) posterior images $\mathbf{I}_{k+1}$;
	\item
	update the sequence of actions with the action that yields the image (among all $\mathbf{I}_{k+1}$) which is the closest to $\mathbf{I}_*$ and set this image as the $\mathbf{I}_k$ for the following step; 
	\item
	loop until none of the posterior images $\mathbf{I}_{k+1}$ is closer to $\mathbf{I}_*$ than $\mathbf{I}_{k}$ -- in other words, loop while~(\ref{eq:errorDecrease}) is verified;
	\item
	output the action sequence $\mathbf{s}_*$.
%	\item 
%	otherwise, move to next iteration $k$, by propagating the image $\mathbf{I}_{k}$ which is the nearest (in terms of $d$) to $\mathbf{I}_*$, to all its posterior images $\mathbf{I}_{k+1}$
%	\item
%	since the tree will grow very quickly (it will have $\left(\dim {\mathcal{A}}\right)^{k}$ leaves at iteration $k$), every $k_{max}$ iterations ``prune'' it by removing all images, except those along the sequences leading to the best $p$ images (best in terms of distance $d$).	
\end{enumerate}
The most challenging part of the algorithm is the prediction of the effect of action $\mathbf{a}^i$ on the image (lines 6 to 8). This requires an appropriate choice of $\mathbf{f}$, among the three implementations proposed in Sec.~\ref{sec:predPCloud}.

\begin{algorithm}
	\caption{Action sequence planner}
	\begin{algorithmic}[1] 
		\renewcommand{\algorithmicrequire}{\textbf{Input:}}
		\renewcommand{\algorithmicensure}{\textbf{Output:}}
		\REQUIRE Initial image $\mathbf{L}_0$ and final desired image $\mathbf{L}_*$.
		\ENSURE Sequence of actions $\mathbf{s}_* =  \left\{ \mathbf{a}_1, \dots, \mathbf{a}_K \right\}$.
		\STATE initialize $k = 0$, $\mathbf{s} = \emptyset$
		\STATE apply~(\ref{eq:newMat}) to obtain $\mathbf{I}_{0|\emptyset}$ from $\mathbf{L}_{0}$ and $\mathbf{I}_*$ from $\mathbf{L}_*$
		\STATE compute $d_0 = d\left(\mathbf{I}_{0|\emptyset}, \mathbf{I}_* \right)$ with~(\ref{eq:distNewImg})
		\LOOP
		\FOR {each action $\mathbf{a}^{ij} \in {\mathcal{A}}$} 
		\STATE extract submatrix $\underline{\mathbf{I}}$ from $\mathbf{I}_{k|\mathbf{s}}$ via~(\ref{eq:predStep1})
		\STATE using~(\ref{eq:lastMapping}) predict effect of $\mathbf{a}^{ij}$ on $\underline{\mathbf{I}}$: $\underline{\mathbf{I}}_{\mathbf{a}^{ij}}$ 
		\STATE with~(\ref{eq:predStep3}) inject $\underline{\mathbf{I}}_{\mathbf{a}^{ij}}$ in $\mathbf{I}_{k|\mathbf{s}}$ to get $\mathbf{I}_{k+1|\left\{ {\mathbf{s}}, \mathbf{a}^{ij} \right\}}$
		\STATE compute $d_{ij} = d\left(\mathbf{I}_{k+1|\left\{ {\mathbf{s}}, \mathbf{a}^{ij} \right\}}, \mathbf{I}_* \right)$ with~(\ref{eq:distNewImg})
		\ENDFOR
		\IF {$ \exists \; d_{ij} < d_k$}
		\STATE find best action $\mathbf{a}_k = \underset{\mathbf{a}^{ij} \in {\mathcal{A}}}{\mathrm{argmin}}\; \; d_{ij}$ 
		\STATE update sequence ${\mathbf{s}} = \left\{ {\mathbf{s}}, \mathbf{a}_k \right\}$ 
		\STATE $k = k+1$
		\STATE $d_k = d\left(\mathbf{I}_{k| {\mathbf{s}} }, \mathbf{I}_* \right)$
		\ELSE
		\RETURN $\mathbf{s}_* = \mathbf{s}$
		\ENDIF
		\ENDLOOP
	\end{algorithmic}
\end{algorithm}

\subsection{Reinsertion of the ROI}

Interested readers may find more advanced methods (often based on motion planning or on machine learning) in the literature on automatic action selection. %Although we did not implement any of these methods, it is worth reviewing a few of them here. 
A planning method for finding intermediate states to shape a flexible wire is presented in~\cite{KaMo:06}. The authors of~\cite{dulac2015deep} use deep reinforcement learning to choose within a finite number of actions, whereas~\cite{laskey2017learning} relies on imitation learning to choose actions for cloth manipulating.

\section{Experiments}

In this section, we present the experiments that we did to assess our framework. The software and a video of the experiments are available online at: \url{https://github.com/gursoyege/lirmm-farine}. The video is also attached to the paper and available at \url{https://youtu.be/fLGf7c3cbKk}.

\subsection{Setup}

Figure~\ref{fig.bazar} shows our experimental setup. We focus on flour manipulation, with the three actions grasp, knock and poke shown -- along with their parameters -- in Table~\ref{table:Act}. We did not apply actions push and pinch, since the robot actuation restrains it from realizing these actions. The three actions can be applied in $N=15$ positions; hence, for these experiments, the total number of actions is $\dim \mathcal{A} = 3 N = 45$. Flour is placed within a custom-made wooden box of size $400$ x $500$ mm. We use the right arm (KUKA LWR 4+) and Shadow Dexterous Hand of the BAZAR robot~\cite{j20} with an Intel Realsense D435 depth camera rigidly linked to BAZAR's right wrist. The camera parameters obtained after calibration are $u_0=421$, $v_0=243$, $f_x=433$, $f_y=433$. Although the image resolution is $848 \times 480$, to calculate $d$ we have focused on the region of interest containing the box and flour, which has size $480 \times 395$. 

%\textbf{est-ce que j'ajoute ici quelques mots sur le fonctionnement de rkcl pour faire du contrôle multi-robots? (plusieurs threads pour les boucles driver et une boucle de contrôle pricipale)}

% To deal with the fact that the arm and the hand are actually distinct hardware components, RKCL integrates a parallelization strategy...

The actions are realized using the RKCL framework\footnote{\url{https://gite.lirmm.fr/rkcl/rkcl-core}}. This C++ library allows to quickly set up a kinematic controller for any kind of single and multi robot structures, as long as the resulting system has a kinematic tree structure. With RKCL, the BAZAR components used in this experiment (one KUKA LWR 4+ arm and the attached Shadow hand) can be controlled in the operational or in the joint space as a unique entity. %By providing a kinematic model of the robot, RKCL allows to set up and achieve operations in the Cartesian space. 
To this end, RKCL integrates a Quadratic Programming controller for tracking a desired pose with the robot tool, while respecting some constraints (e.g., joint position/velocity/acceleration limits, restrictions in the Cartesian space to avoid collisions, tool velocity/acceleration). 

%The library offers a convenient way to easily set up a new scenario using configuration files written in YAML data-serialization language\footnote{\url{https://yaml.org/}}. A YAML file, loaded at the program initialization, allows to configure the global properties of the controller which hold for the entire robot performance. These properties include the joints involved in robot motion, the robot point to be controlled, the set of constraints, the control gains, or the control loop sampling time. A second YAML file defines the sequence of operations the robot should perform to achieve the whole action. Note that in RKCL it is possible to split the set of joints in different groups, and choose to control each group independently, either in joint or in operational space. For each operation, the user can specify the control mode for each group. For instance, a joint group may regulate the motion either in the Cartesian space, or in the joint space. 

\begin{figure}[t]
	\centering {\centering\includegraphics[width=0.85\columnwidth]{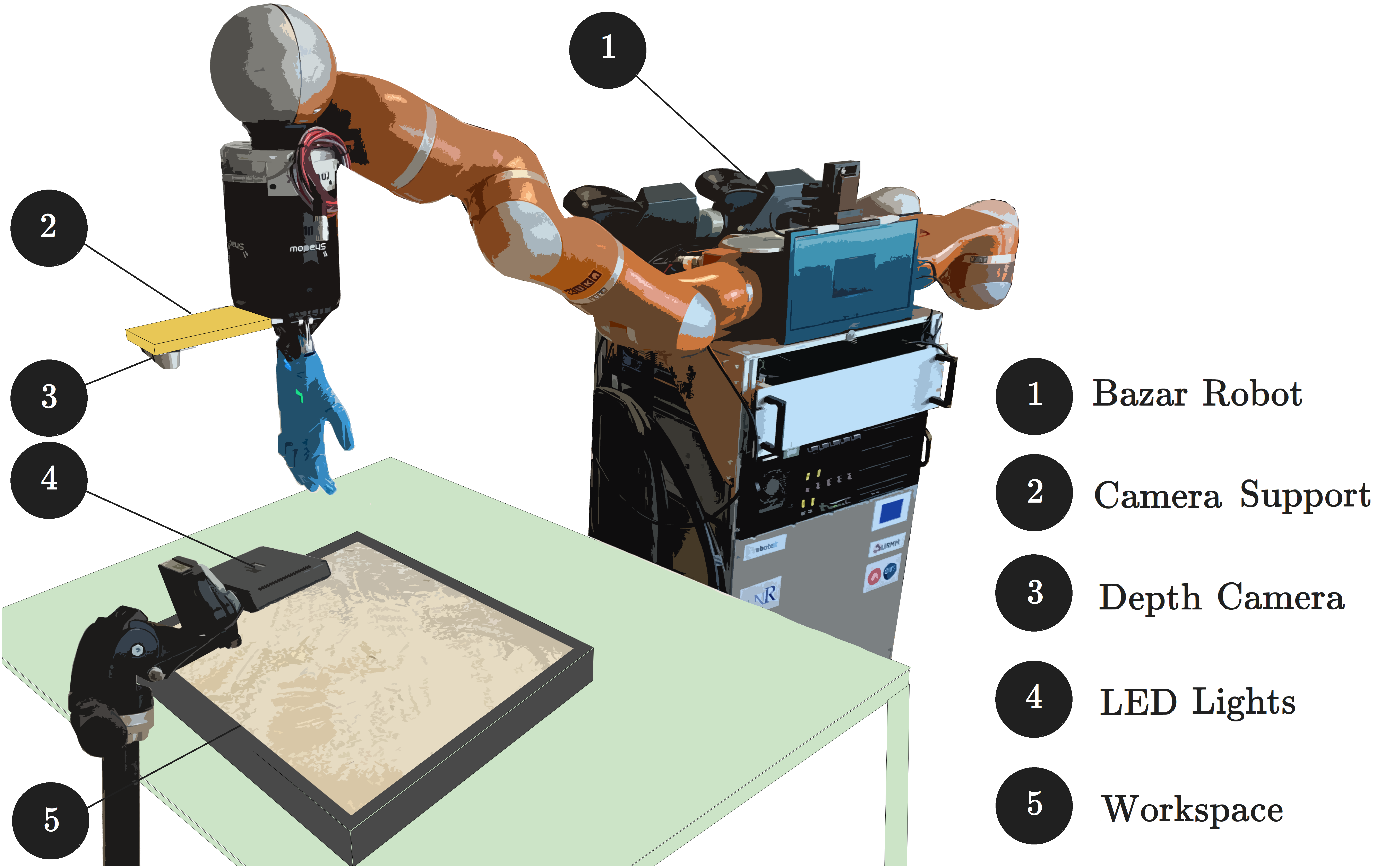}}
	\caption{Experimental Setup for robotic flour shaping.}
	\label{fig.bazar}
    \vspace{-0.2cm}
\end{figure}

%For point cloud shaping, we have divided the set of active joints into two groups: the first includes the full joints of the arm plus the wrist joints of the hand. Setting this group in Cartesian space control mode enables us to monitor the pose of the hand palm in the workspace. The second group consists of the hand finger joints. This groups operates in joint space to carry out flour manipulations in a precise and repeatable manner. 

\begin{figure*}[!h]
	\centering
	\begin{subfigure}[b]{0.16\textwidth}
		\label{fig_grasp1}
		\includegraphics[width=\textwidth]{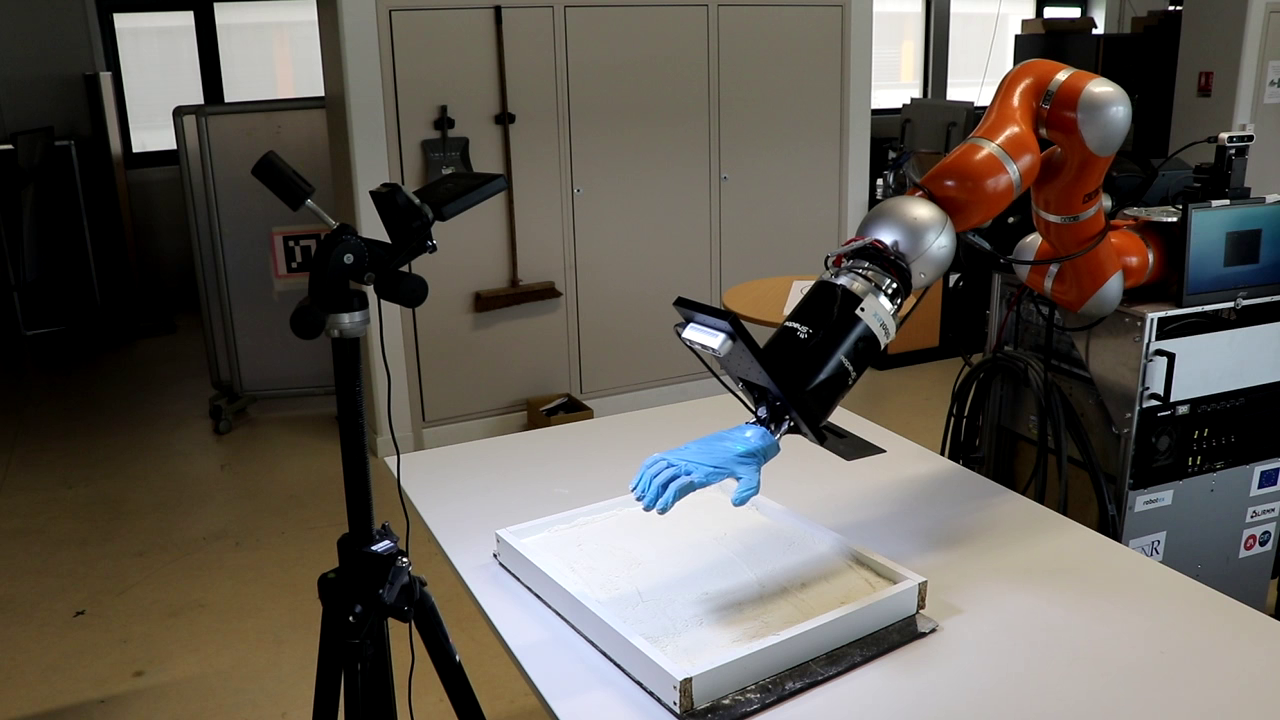}
	\end{subfigure}
	\hfill
	\begin{subfigure}[b]{0.16\textwidth}
		\label{fig_grasp2}
		\includegraphics[width=\textwidth]{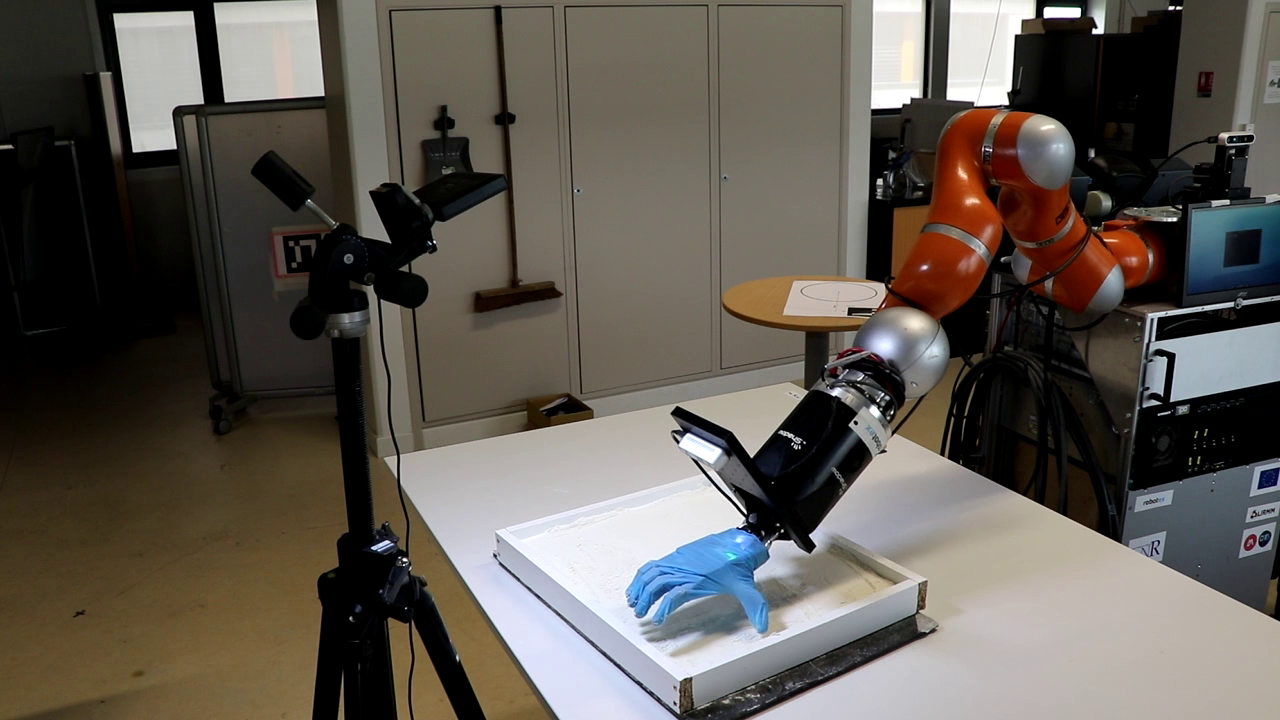}
	\end{subfigure}
	\hfill
	\begin{subfigure}[b]{0.16\textwidth}
		\label{fig_grasp3}
		\includegraphics[width=\textwidth]{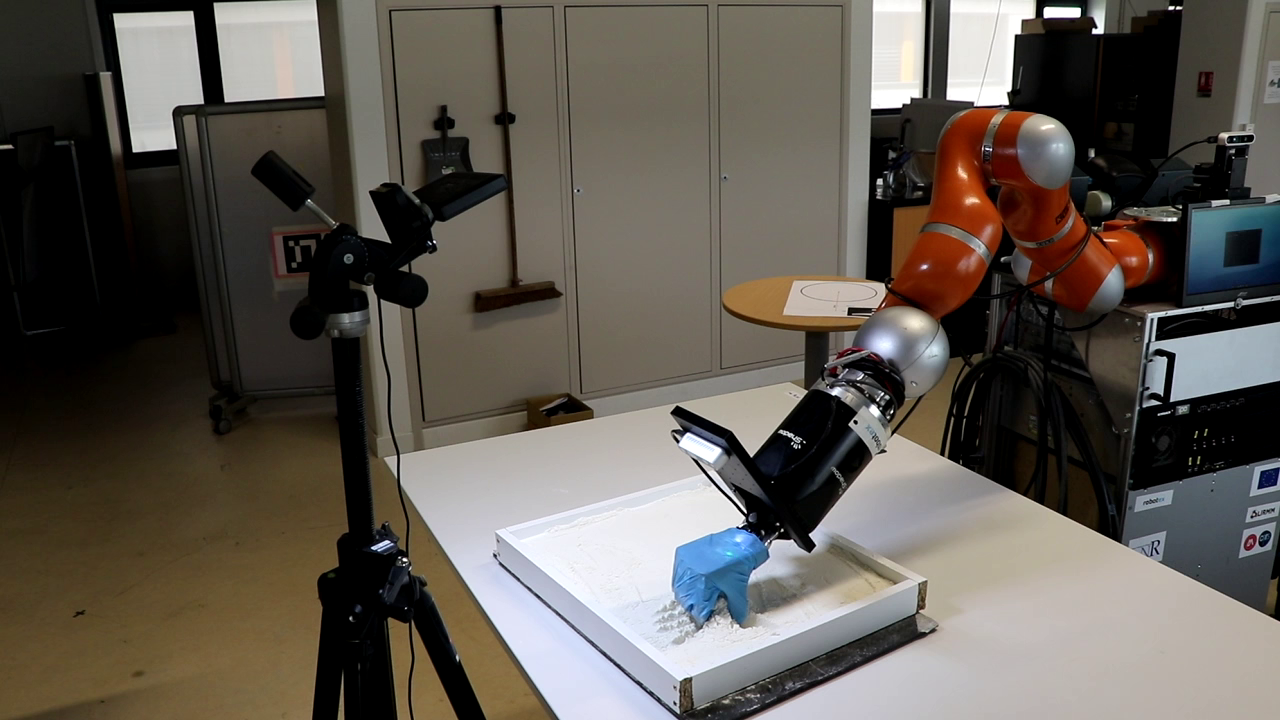}
	\end{subfigure}
	\hfill
	\begin{subfigure}[b]{0.16\textwidth}
		\label{fig_grasp4}
		\includegraphics[width=\textwidth]{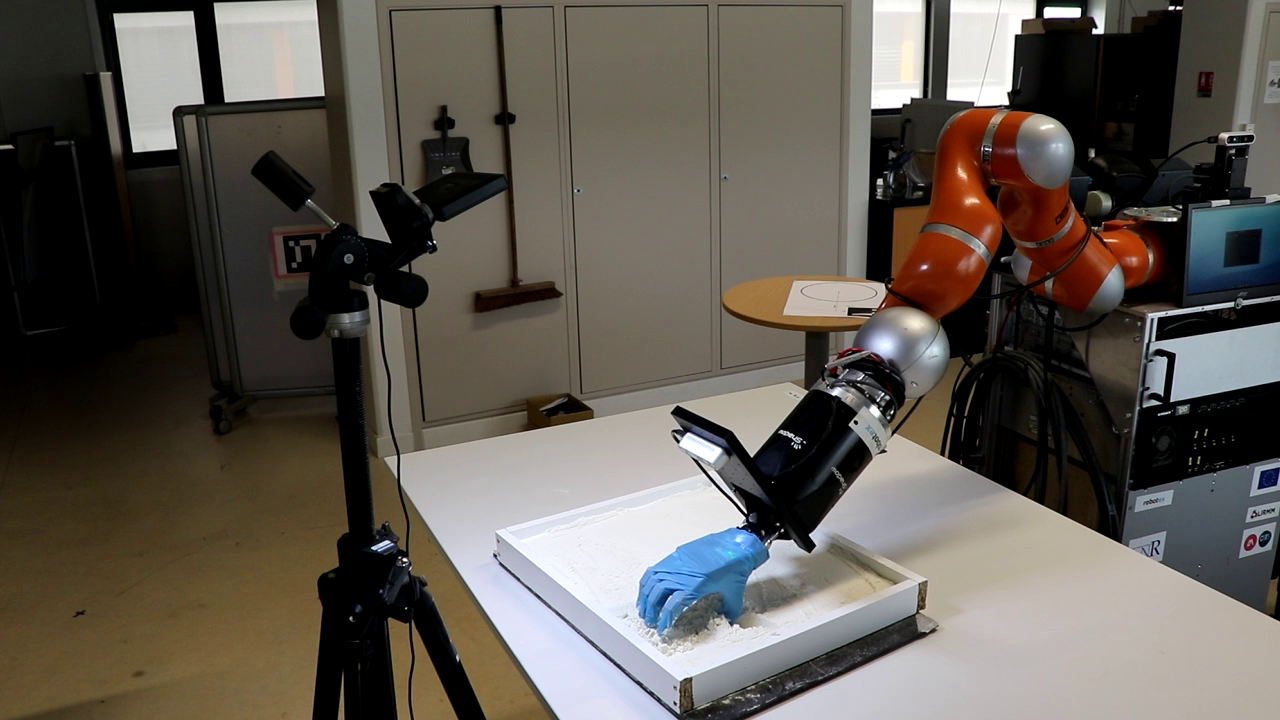}
	\end{subfigure}
	\hfill
	\begin{subfigure}[b]{0.16\textwidth}
		\label{fig_grasp5}
		\includegraphics[width=\textwidth]{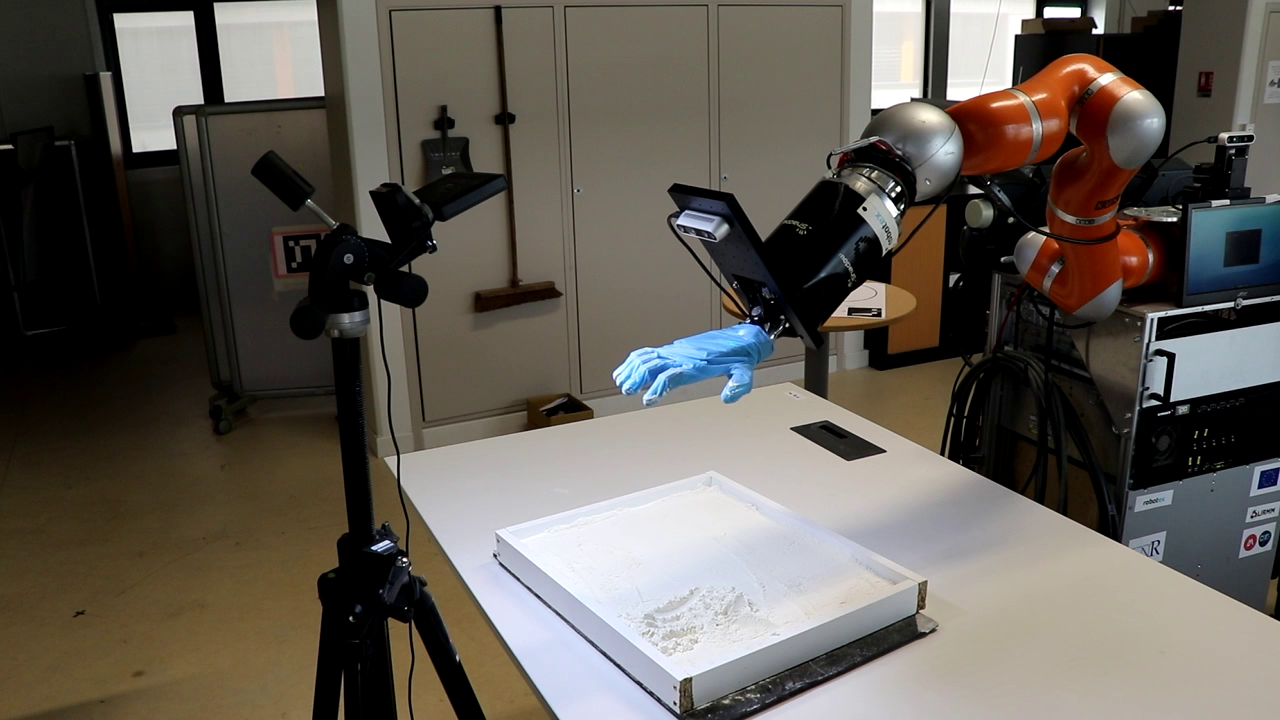}
	\end{subfigure}
	\hfill
	\begin{subfigure}[b]{0.16\textwidth}
		\label{fig_grasp6}
		\includegraphics[width=0.95\textwidth]{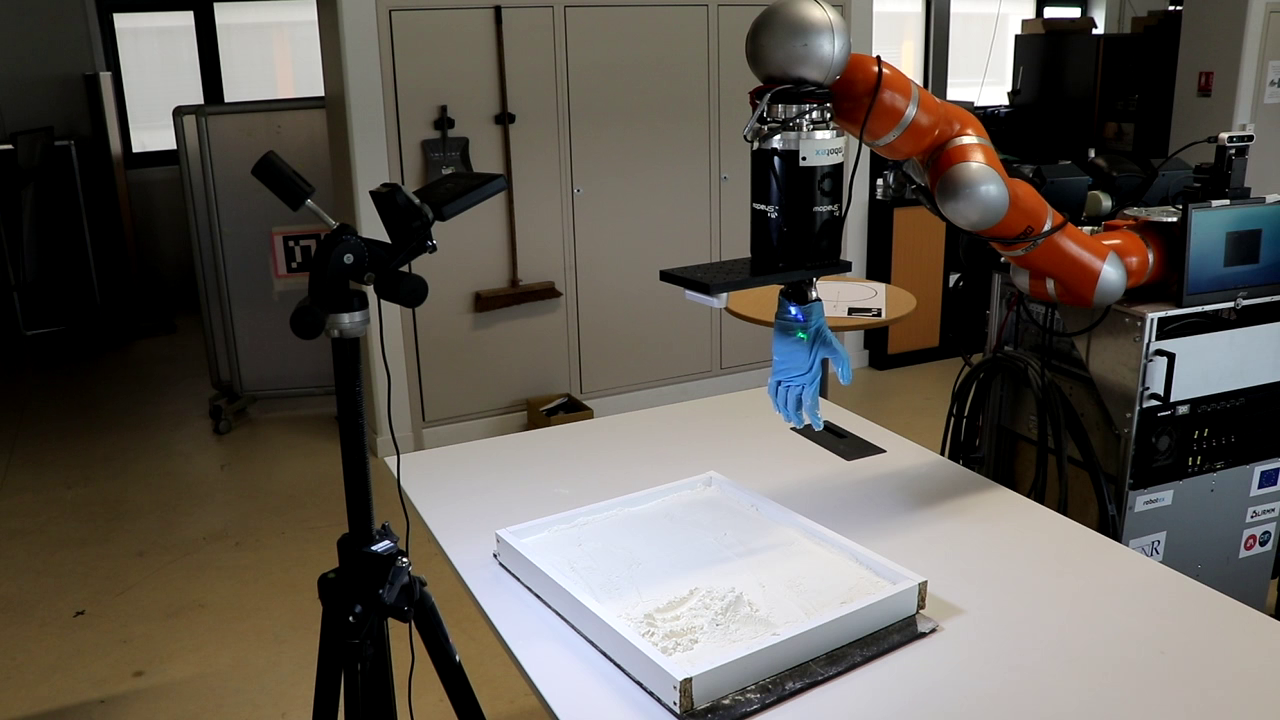}
	\end{subfigure}
	\caption{Six snapshots taken during the grasp action. Left to right: the arm moves above the desired location; the arm goes down, while the hand opens; the hand closes to grasp some flour; the hand opens to release the flour; the arm goes up; the arm returns to the starting position, to free the camera field of view.}
	\label{fig_xp_snap}
    \vspace{-0.2cm}
\end{figure*}

The grasping action is depicted in Fig.~\ref{fig_xp_snap}: 
the robot performs six successive operations to accomplish the whole action. A similar approach has been taken for the two other actions. To safely manage contact with the environment, we monitor the KUKA LWR's joint torques and switch to the next operation if they pass a given threshold. After each action, the arm returns to a fixed pose such that the camera is perpendicular to the workspace. There, the Realsense captures a depth image, at a distance of approximately $450$ mm from the flour. This value corresponds to the depth of all actions' positions: $Z = 450$ mm. Since it is much greater than the depth of the three bounding boxes, $\Delta Z = 50$ mm (see Table~\ref{table:Act}), \underline{Hypothesis 3} is valid.

%This guarantees that the depth of the bounding box $2 \Delta Z$ stays roughly the same and small - independently from the position of the sand wrt robot. Alternatively, one could have driven the action using $Z$ measured by the depth image as a feedback signal. 

\begin{figure}[t!]
    \vspace{-0.2cm}
	\centering
	\begin{subfigure}[b]{0.24\textwidth}
		\centering
		\includegraphics[width=\textwidth]{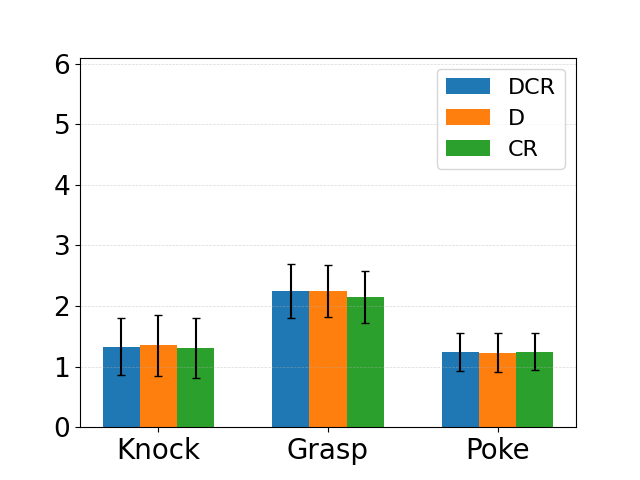}
		\caption{\small Train on R and test on R.}    
		\label{fig:RonR}
  \vspace{-0.4cm}
	\end{subfigure}
	\hfill
	\begin{subfigure}[b]{0.24\textwidth}  
		\centering 
		\includegraphics[width=\textwidth]{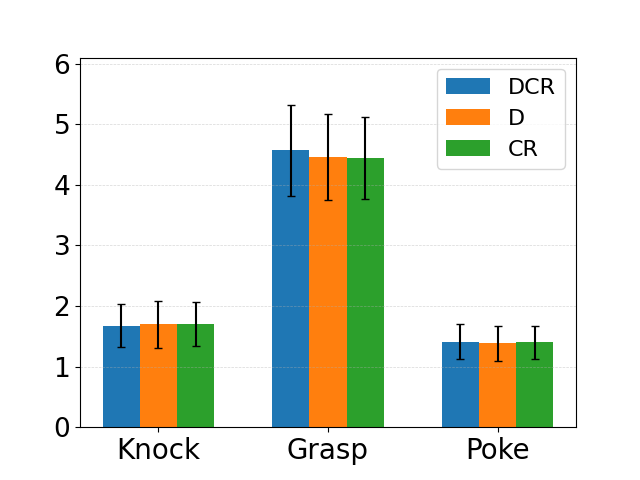}
		\caption{\small Train on R and test on H.}    
		\label{fig:RonH}
   \vspace{-0.4cm}
	\end{subfigure}
	\vskip\baselineskip

	\begin{subfigure}[b]{0.24\textwidth}   
		\centering 
		\includegraphics[width=\textwidth]{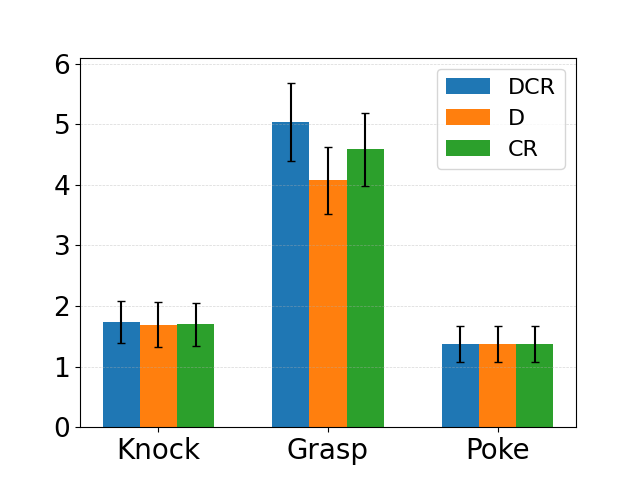}
		\caption{\small Train on H and test on R.}   
		\label{fig:HonR}
	\end{subfigure}
	\hfill
	\begin{subfigure}[b]{0.24\textwidth}   
		\centering 
		\includegraphics[width=\textwidth]{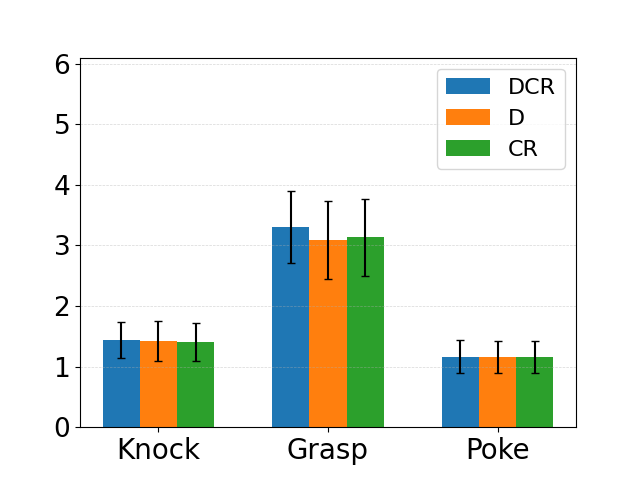}
		\caption{\small Train on H and test on H.}    
		\label{fig:HonH}
	\end{subfigure}
	\caption {Mean and standard deviation of the distances (in mm) between the predicted and true ROI, when using $\mathbf{f}_\mathrm{D}$ (orange), $\mathbf{f}_\mathrm{CR}$ (green) or $\mathbf{f}_\mathrm{DCR}$ (blue). The four graphs are obtained by training and testing on either human or robot data (denoted respectively H and R).} 
	\label{fig.errBars}
    \vspace{-0.4cm}
\end{figure}

We record sequences of depth images $\mathbf{L}$ of the flour, while the actions are applied by either BAZAR or a human. Images are acquired by Intel RealSense SDK 2.0, with temporal filtering and hole filling enabled\footnote{See post-processing filters in Intel RealSense SDK 2.0: \url{https://github.com/IntelRealSense/librealsense}}. We remove all images where either the robot or human are occluding the environment (typically, while the robot/human is moving to execute the action). We apply~(\ref{eq:newMat}) to convert each image $\mathbf{L}$ into $\mathbf{I}$. We divide the recorded images into $6$ datasets (one for each of the $3$ actions, applied by either robot or human). All datasets are publicly available at the following link: \url{https://seafile.lirmm.fr/d/bc9b413e83184505aa7e/}. Each dataset contains $100$ pairs of prior/posterior images, representing the flour before/after the action. Then, we manually annotate, in all these images, the ROI which contains a ``visible'' (by human standards) difference in depth\footnote{The annotation could be automated using standard image processing techniques, similar to the ones we developed in~\cite{cherubini:ijrr:2020}.}. For each action, we fix the same ROI sizes for the human and robot, so that the generators can be trained/tested across agents. The sizes $( \underline{w}, \underline{h})$ of the grasp, knock and poke ROI are respectively: $(183,183)$, $(112,60)$, $(40,45)$; examples of these ROI are also shown in Table~\ref{table:Act}.

To implement both generators ${\mathbf{f}}_\mathrm{C}$ and ${\mathbf{f}}_\mathrm{DC}$, we use Tensorflow 2 \footnote{\url{https://www.tensorflow.org}} on a Dell Precision 7550 laptop (Intel Xeon W-10885M CPU, 64GB RAM, NVIDIA RTX 5000). We train both networks with the Adam optimizer~\cite{adamoptimizer}, setting $\alpha=0.0002$, $\beta_1=0.5$, $\beta_2=0.999$ and 
$\epsilon=10^{-8}$. 

\subsection{Results of Image Prediction}
\label{Sec.ResultsImagePrediction}

First, to assess our distance $d$, we have compared the time needed to compute the distance between a pair of depth images, using: our distance, Chamfer distance and Earth Mover's Distance (EMD). The results (on Dell Precision 7550 mentioned above) are: our distance takes $17$ ms, Chamfer $153$ ms and EMD $157$ ms\footnote{We used the following implementations: \url{https://www.tensorflow.org/graphics/api_docs/python/tfg/nn/loss/chamfer_distance} and \url{https://en.wikipedia.org/wiki/Earth_mover's_distance##PyTorch_Implementation}, for Chamfer and EMD (respectively).}. 
These results confirm our choice. It should also be noted that both Chamfer and EMD rely on point clouds, with the drawbacks on memory usage mentioned above.

Then, we realized a series of experiments to assess the functions ${\mathbf{f}}_\mathrm{D}$, ${\mathbf{f}}_\mathrm{CR}$ and ${\mathbf{f}}_\mathrm{DCR}$. We split each of the $6$ datasets into training+validation ($75 \%$) and testing ($25 \%$) subsets. For cross-validation, the partitions are done four times, each time selecting randomly -- as test subset -- a different $25 \%$ of the whole dataset. The labeled (training+validation) set ${\mathcal{G}}$ is used to: train generators ${\mathbf{f}}_\mathrm{C}$ and ${\mathbf{f}}_\mathrm{DC}$, find the average difference $\Delta \underline{\mathbf{I}}$, boolean matrix $\underline{\mathbf{B}}$, and mean offsets $\mu_\mathrm{C}$ and $\mu_\mathrm{D}$. We train a different generator: for each action (three), agent (two: human or robot), task (two: predicting with ${\mathbf{f}}_\mathrm{C}$ or denoising with ${\mathbf{f}}_\mathrm{DC}$), and partition (four) adding up to a total 48 networks. For each network, we set the training duration to 500 epochs (approximately 1 hour).

Figure~\ref{fig.errBars} shows the distances (in mm) between the predicted and true ROI (mean and standard deviation of the 4 tests) for each action and for each choice of ${\mathbf{f}}$ (either ${\mathbf{f}}_\mathrm{D}$, ${\mathbf{f}}_\mathrm{CR}$ or ${\mathbf{f}}_\mathrm{DCR}$). The distances are obtained by expressing distance $d$ according to~(\ref{eq:distNewImg}), and then converting it to mm. We repeat the measures four times (graphs a-d) by training and testing on either human or robot.

We can draw various conclusions from these graphs. First, testing after having trained with the same agent (Figures~\ref{fig.errBars}a and~\ref{fig.errBars}d) provides lower error than transferring from one agent to the other (Figures~\ref{fig.errBars}b and~\ref{fig.errBars}c); on average, $1.75\pm0.42$ mm versus $2.53\pm0.46$ mm. This is because the images of human and robot actions are quite different, also due to the hand sizes. Yet, the transfer learning error is low enough to encourage some of the real robot tests, which we will present in the next Section. Second, training and testing on robot ($1.59\pm0.41$ mm) outperforms training and testing on human ($1.92\pm0.43$ mm), likely because of the machine's lower variability. Third, independently from the agent and prediction strategy ${\mathbf{f}}$, the error (both mean and standard deviation) tends to increase from poke to knock to grasp. This is because of their difference in terms of repeatability (e.g., the effect of a poke, which involves only one finger, varies less than that of a grasp). Last, but not least, there is no relevant difference between the mean and standard deviation of the three prediction functions, $\mathbf{f}_\mathrm{D}$ (orange), $\mathbf{f}_\mathrm{CR}$ (green) or $\mathbf{f}_\mathrm{DCR}$ (blue). Overall, the three functions perform similarly: $\mathbf{f}_\mathrm{D}$ with error $2.09 \pm 0.44$ mm, $\mathbf{f}_\mathrm{CR}$ with error $2.13 \pm 0.44$ mm, $\mathbf{f}_\mathrm{DCR}$ with $2.21 \pm 0.44$ mm. 

To better discriminate between the three prediction functions, we have run tests with training+validation subsets of $5$, $10$, $25$, $50$ and $75$ images, while keeping the same testing subset of $25$ images. We consider robot data for both training and testing and average over all three actions. The results -- plotted in Fig.~\ref{fig.dataSize} -- show that as the dataset increases, the deep learning methods $\mathbf{f}_\mathrm{DCR}$ and $\mathbf{f}_\mathrm{CR}$ decrease the error faster than the more na\"ive $\mathbf{f}_\mathrm{D}$. This is likely because, with more data, the image output by $\mathbf{f}_\mathrm{D}$ tends to blur, while the generator embedded in $\mathbf{f}_\mathrm{DCR}$ and $\mathbf{f}_\mathrm{CR}$ succeed in capturing the depth image details. 

Finally, we have assessed the importance of the refining step, by comparing the prediction error with and without refinement. More specifically, we have compared the errors of ${\mathbf{f}}_\mathrm{C}$, ${\mathbf{f}}_\mathrm{CR}$, ${\mathbf{f}}_\mathrm{DC}$ and ${\mathbf{f}}_\mathrm{DCR}$ in Table~\ref{table:refin}. The results clearly show the usefulness of the refinement step (values in the third and fifth column are all lower than those in the second and fourth column).

\begin{figure}[h!]
    \vspace{-0.3cm}
	\centering {\centering\includegraphics[trim=0 3cm 0 2cm, clip,width=0.9\columnwidth]{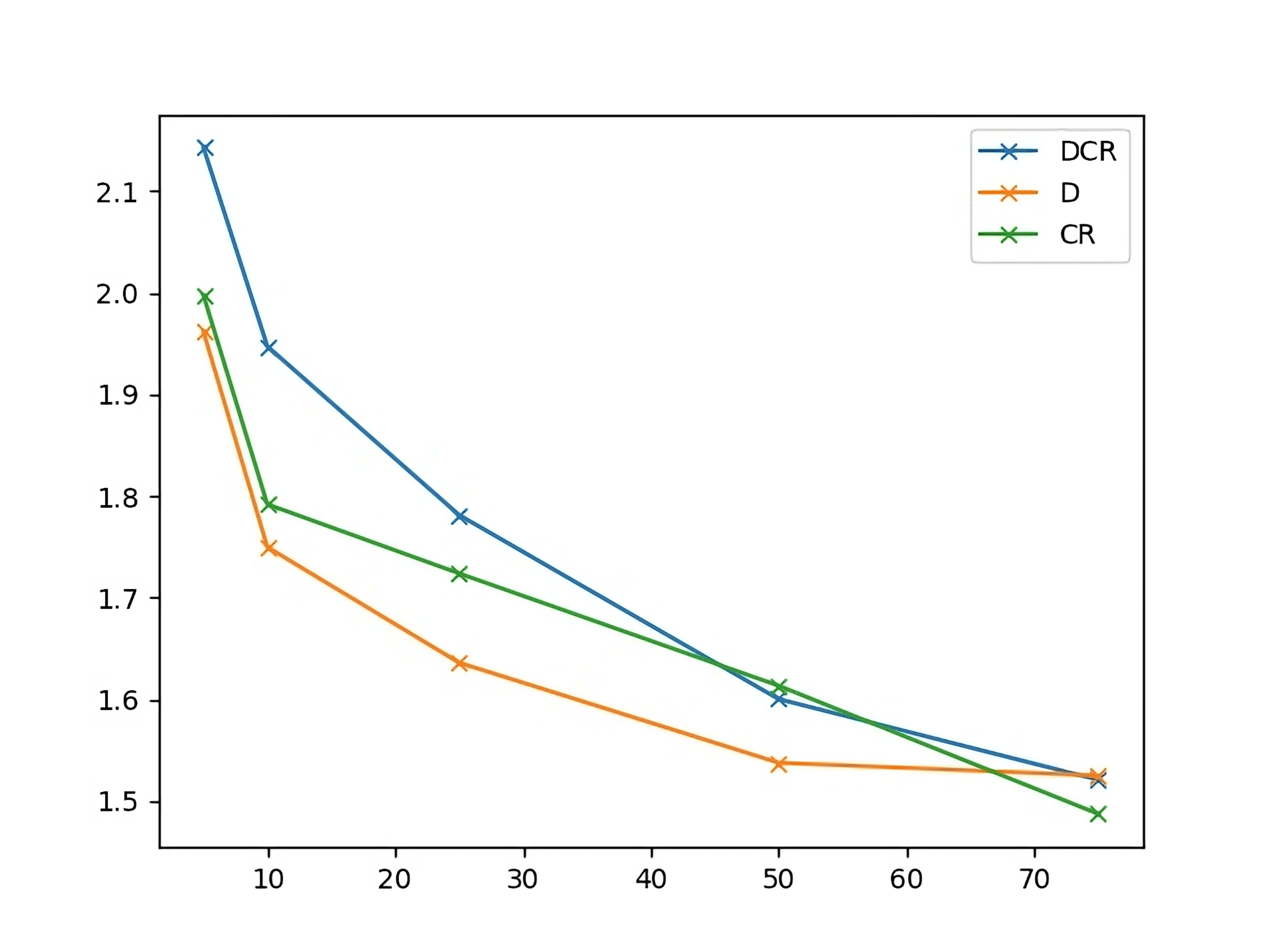}}
	\caption{Mean error (in mm) of $\mathbf{f}_\mathrm{D}$ (orange), $\mathbf{f}_\mathrm{CR}$ (green) or $\mathbf{f}_\mathrm{DCR}$ (blue) over the three actions as the training+validation dataset increases: $5$, $10$, $25$, $50$ and $75$ images.}
	\label{fig.dataSize}
    \vspace{-0.2cm}
\end{figure}

\begin{table}[h!]
	\caption{Mean error (in mm) of ${\mathbf{f}}_\mathrm{C}$, ${\mathbf{f}}_\mathrm{CR}$, ${\mathbf{f}}_\mathrm{DC}$ and ${\mathbf{f}}_\mathrm{DCR}$.} 
	\label{table:refin}
	\centering
	\begin{tabular}{|c|c|c|c|c|}
		\cline{1-5}
\textbf{Datasets} & ${\mathbf{f}}_\mathrm{C}$ & ${\mathbf{f}}_\mathrm{CR}$ & ${\mathbf{f}}_\mathrm{DC}$ & ${\mathbf{f}}_\mathrm{DCR}$\\ 
		\hline
		Train on R and test on R  & 1.65 & 1.56 & 1.73 & 1.60 \\
		\hline
		Train on R and test on H  & 2.68 & 2.51 & 2.80 & 2.55 \\
		\hline
		Train on H and test on R  & 2.70 & 2.55 & 2.95 & 2.71 \\
		\hline
		Train on H and test on H  & 1.96 & 1.90 & 2.11 & 1.97 \\
		\hline
	\end{tabular}
 \vspace{-0.4cm}
\end{table}

\subsection{Robotic point cloud shaping}
\label{sec:expRob}
Since among the three prediction functions, $\mathbf{f}_\mathrm{CR}$ appears to be the most promising as data increases (see Fig.~\ref{fig.dataSize}), we have used it in the following experiments to predict the effect of an action on a region of interest.

We have run two series of experiments: the first four with robot-generated images and the next two with human-generated images. In both series, for each of the three actions, we take the best network $\mathbf{f}_\mathrm{C}$ among the four dataset partitions, and then proceed as follows. Given a pair of images ($\mathbf{L}_0$ of the initial point cloud $\mathcal{P}_0$, and $\mathbf{L}_*$ of the final desired point cloud $\mathcal{P}_*$), we run \textbf{Algorithm 1} (with $\mathbf{f}_\mathrm{CR}$ applied at line 7) to obtain the sequence of actions $\mathbf{s}_*$ for shaping the point cloud from $\mathcal{P}_0$, to $\mathcal{P}_*$. Then, we control BAZAR robot, so that it sequentially executes the actions in $\mathbf{s}_*$, starting from point cloud $\mathcal{P}_0$. The final desired depth images $\mathbf{L}_*$ have been shaped by the robot in the first series of experiments, and by the human in the second series. 

\begin{figure}[b!]
	\centering {\centering\includegraphics[width=0.95\columnwidth]{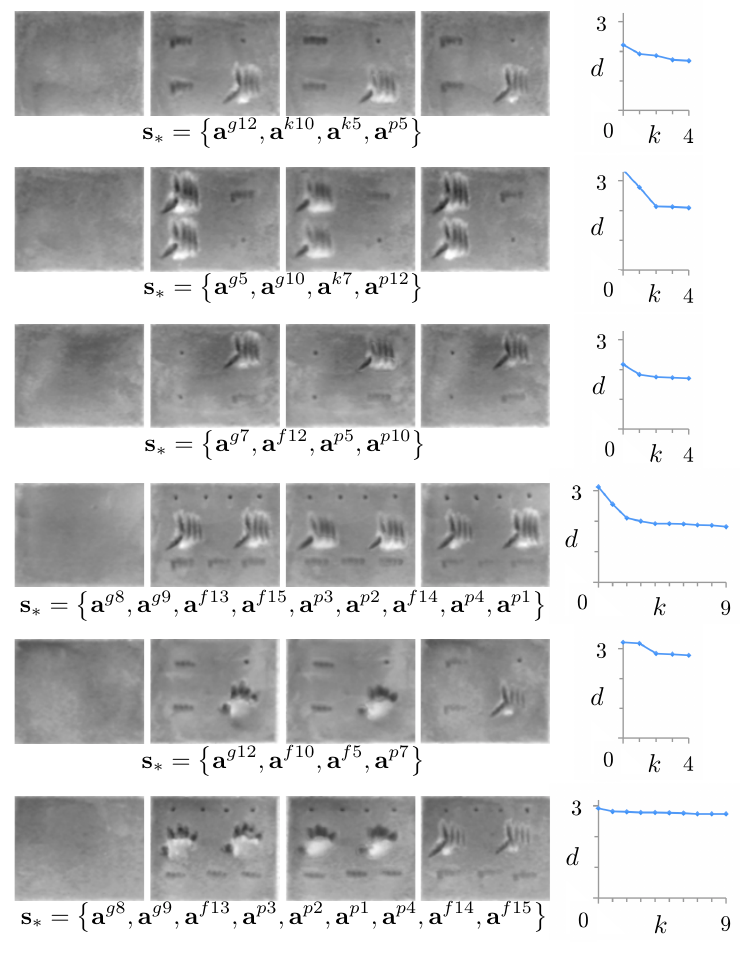}}
	\caption{Six robot experiments, with the neural networks trained on the robot dataset (top four) and on the human dataset (bottom two). For each of the six experiments, we show four images (left to right: initial image $\mathbf{L}_0$, final desired image $\mathbf{L}_*$, predicted image, and image obtained by the robot) as well as the computed sequence $\mathbf{s}_*$ and the evolution of $d$ (in mm) after each action $\mathbf{a}_k$ is applied.}
	\label{fig.robExperiments}
    \vspace{-0.2cm}
\end{figure}

The results are shown in Fig.~\ref{fig.robExperiments}. For each of the six experiments, we show (from left to right) the initial image $\mathbf{L}_0$, the final desired image $\mathbf{L}_*$, the image predicted by \textbf{Algorithm 1}, and the image obtained by the robot; we also indicate the computed sequence $\mathbf{s}_*$ and the evolution of $d$ (in mm) over the experiments. In all six experiments, \textbf{Algorithm 1} retrieves the actions which were applied by the agent to obtain $\mathbf{L}_*$, although not necessarily in the same order. As the robot replays the sequence, the distance diminishes at each step in all experiments (see blue curves). As expected, the final distance is higher in the transfer learning experiments (both dataset and desired image are human-made) than in the entirely robotic experiments. This is due to the difference in shape and size of the human and robot actions, which is visible by comparing the fourth image to the second and third images, in the last two experiments. 
%The final distance (after the robot experiment) is always higher than the distance between the image predicted by \textbf{Algorithm 1} and $\mathbf{L}_*$. 

%We have also run a few tests where $\mathbf{L}_*$ is obtained by applying two or more actions at the same position $\mathbf{t}$. These tests failed, since \textbf{Algorithm 1} could not discern between the superposed actions. The reason is that our current datasets do not have images of superposed actions. In general, all images in the datasets have been obtained by applying the actions on flat (parallel to the image plane) flour. This is one of the many limits of our work, which we will discuss in the next Section. 

\subsection{Generalizing across actions, materials, and geometry}

Given the limitations of the robot hardware, we decided to run another series of experiments on human data, to validate our methodology in more general situations. In particular, we assessed the following aspects:
\begin{itemize}
\item
applying any among all five actions of Table~\ref{table:Act} (including pinch and press),
\item
shaping new -- non granular -- material with the functions ${\mathbf{f}}_\mathrm{D}$, ${\mathbf{f}}_\mathrm{CR}$ trained on flour; to this end, we used kinetic sand, a plastic toy material that mimics the physical properties of wet sand,
\item
experimenting action superposition (i.e., situations where  multiple actions affected the same region),
\item
shaping non-flat surfaces  (both irregular and of constant, non-null curvature).
\end{itemize}

 \begin{figure}[b!]
	\centering {\centering\includegraphics[width=0.95\columnwidth]{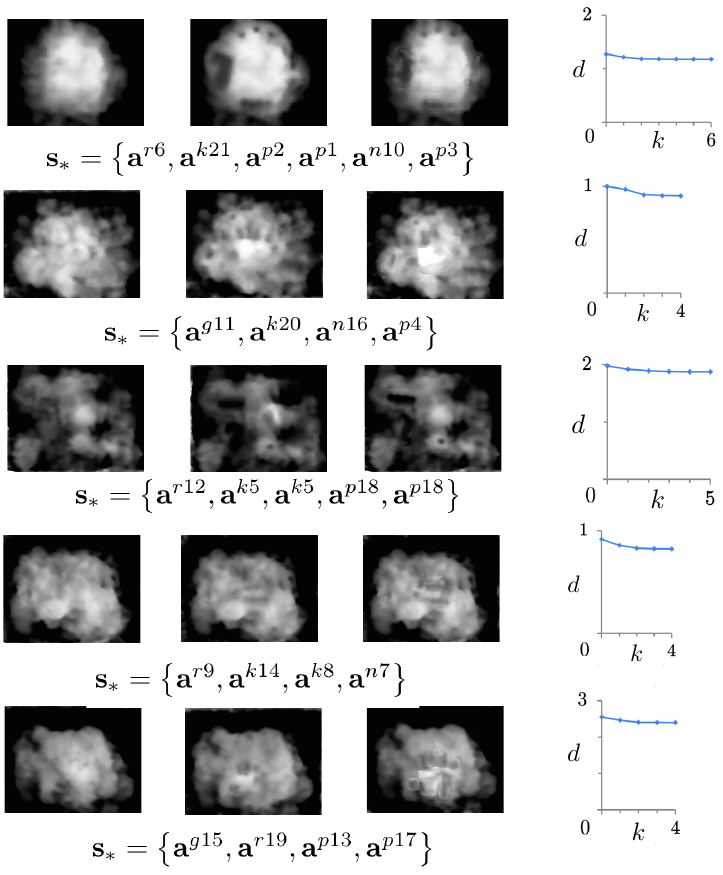}}
	\caption{Five experiments, with the neural networks trained on the human dataset. For each experiment, we show three images (left to right: initial image $\mathbf{L}_0$, final desired image $\mathbf{L}_*$ and predicted image) as well as the computed sequence $\mathbf{s}_*$ and the evolution of $d$ (in mm) after each action $\mathbf{a}_k$ is applied.}
	\label{fig.expRevision}
    \vspace{-0.2cm}
\end{figure}

We have run five experiments on kinetic sand, shown in Fig.~\ref{fig.expRevision}, with $\mathbf{f}_\mathrm{CR}$ as prediction function. The actions can be applied in $N=21$ positions; hence, for these experiments, the total number of actions is $\dim \mathcal{A} = 5 N = 105$. In all experiments and for each of the five actions, we take the best network $\mathbf{f}_\mathrm{C}$ among the four human dataset partitions, and then proceed as follows. Given a pair of images ($\mathbf{L}_0$ of the initial point cloud $\mathcal{P}_0$, and $\mathbf{L}_*$ of the final desired point cloud $\mathcal{P}_*$), we run \textbf{Algorithm 1} (with $\mathbf{f}_\mathrm{CR}$ applied at line 7) to obtain the sequence of actions $\mathbf{s}_*$ for shaping the point cloud from $\mathcal{P}_0$, to $\mathcal{P}_*$. Then, we verify if the predicted image qualitatively resembles the final one, and whether \textbf{Algorithm 1} predicts a sequence of actions similar to the actually applied one. 

For each of the 5 experiments, we show (from left to right) the initial image $\mathbf{L}_0$, the final desired image $\mathbf{L}_*$ and the image predicted by \textbf{Algorithm 1}; we also indicate the computed sequence $\mathbf{s}_*$ and the evolution of $d$ (in mm) over the experiments. 

In the first experiment, the initial shape has a constant curvature of approximately $15$ cm. For the four other experiments, we start with very irregular non-flat shapes (with standard variation of the depth from its mean $\pm 17$ mm -- see first column of Fig.~\ref{fig.expRevision}). In the fourth and fifth experiments, we obtained the desired image by apply overlapping actions.

In four of the five experiments, \textbf{Algorithm 1} retrieves the same type, position and number of actions which the human applied to obtain $\mathbf{L}_*$, although not necessarily in the same order. The only exception is the third experiment. Here, our method predicts two knock and two poke actions at the same position (see $\mathbf{s}_*$ at the third line of Fig.~\ref{fig.expRevision}), whereas the applied sequence was: $\mathbf{s}_* =  \left\{ \mathbf{a}^{r12}, \mathbf{a}^{k5}, \mathbf{a}^{p18} \right\}$. When applying knock and poke to generate $\mathbf{L}_*$, the human has penetrated the material more than usual. On the human dataset, the deepest points generated by knock and poke have depths respectively $6.80$ mm and $7.05$ mm. Instead, to generate $\mathbf{L}_*$, the human applied a knock and a poke of maximal depths $19.60$ mm and $17.05$ mm, respectively. Therefore, the solution proposed by \textbf{Algorithm 1} is coherent with the task. It should also be noted that in all five experiments, the distance diminishes at each step (see blue curves). 

Finally, we repeated the five experiments using the na\"ive $\mathbf{f}_\mathrm{D}$ function to predict action effect. The results were much less satisfying in terms of both final image quality and distance error. In none of the experiments, \textbf{Algorithm 1} managed to retrieve the type, position and number of actions which the human applied to obtain $\mathbf{L}_*$. This result leads us to believe that the molding problem cannot be solved by such a simple approach, and that particularly in complex scenarios such as the ones studied here, deep learning can provide a much better solution to the action prediction problem.

%Since our method $\mathbf{f}_\mathrm{R}$ regularize the average depth of the output according to the average depth of the dataset, applying the same action at the same position generates \textbf{more} depth.

\section{CONCLUSIONS AND FUTURE WORK}

Although our results are encouraging, they are preliminary -- to say the least. The path towards robotic molding is still long and paved with obstacles. 

First, the current hardware substantially limits the range of robot operation. Even an extraordinary robot like the Shadow Dexterous Hand, used here, is designed for conventional rigid object manipulation, and can by no means reproduce the versatility and adaptability of the human hand. It is rigid, has non-reproducible behavior (due to cable actuation), is weak (e.g., not strong enough to penetrate compressed flour) and tends to break easily. This is also why if was difficult to actuate the fingers, and in two out of three robot actions (all except \textit{grasp}) we maintained a constant hand posture while in contact with the material. We had to dedicate most of our time to find the appropriate material and actions, in order to record relevant data without breaking the robot! Also in terms of perception, the gap between technology and nature is critical. 

Present-day force sensing is far from providing good accuracy at the low ranges required by our application, and it could not be used to close the actions' feedback loop (e.g., to decide when or how to grasp the flour). Our Shadow Hand is equipped with five (one per fingertip) SynTouch LLC Biotac tactile sensors. Yet, these sensors are uncalibrated and only provide information if the fingertip touches the environment, which is rarely the case when molding. Furthermore, to shield the robot mechanics and the Biotac themselves from flour, the hand wears a glove, which biases the Biotac readings. The ideal sensor for molding should cover palm and fingers, and be dirt and waterproof. For all these reasons, although we acknowledge the importance of force sensing, we have only focused on visual feedback.

In the future, solutions to these actuation and perception issues may come from soft~\cite{softGrip:18} and underactuated~\cite{softHand:15} robots or from dense tactile skins~\cite{skins:15}. Another worthy aspect is that humans generally rely on both hands for molding; yet, this introduces other difficulties, beyond the scope of this paper

A second limitation comes from the absence of a physical model, which led us to use deep machine learning. We could improve our deep learning methods by adding data from multiple humans (not just one, like here), relying on better computational resources, or automating image annotation as in~\cite{cherubini:ijrr:2020}. Yet, our results show that deep learning outperforms image difference in generalizing across actions, materials and geometry (e.g., irregular surfaces). One of the limitations of the neural networks developed here is that they are not scalable in terms of image size, and therefore lack robustness with respect to variations of the action depth coordinate. Due to their synthetic nature, the predicted images may contain some artifacts and discontinuities. Yet, despite these artifacts, our framework could compute the correct sequence of actions.
%,	\item
% 	they overfit, requiring images of uniform depth as input and impeding prediction when successive actions are applied at the same position,
% 	\item
% 	they require specific data, not as easily available	as that of other deep learning applications (such as image and speech recognition).
% \end{itemize}

A third issue, outlined above, is that our framework relies on too many assumptions, which should gradually be relaxed, to make it applicable in real scenarios. For instance, we cannot afford variations in the relative pose between the camera and workspace, unless a rigorous calibration phase is carried out systematically, to guarantee consistency between the robot operational space and the camera frame. Besides, our action space is limited to predetermined actions. This raises the question of whether our method could provide a reasonable prediction for unknown actions. Furthermore, the action sequence search problem requires the actions to be applied at a limited and known set of positions, making it easily intractable in ambitious scenarios. To generalize to continuous action positions, one could use random sampling, Monte Carlo tree search or particle swarm optimization, to select the best sequence of actions/positions. These approaches are more general than greedy ones, since they explore a larger space of possible actions sequences. Yet, they can be computationally expensive, especially if the search space is large and the prediction model complex (as is the case here). %To mitigate this, we can use techniques such as pruning, parallelization, or approximation to speed up the search.

%To elaborate the latter, we can formulate the search problem as an optimization problem and use optimization algorithms such as gradient descent or genetic algorithms. Another approach could be sampling a set of candidate sequences of actions and use random sampling, or more sophisticated methods such as Monte Carlo tree search or particle swarm optimization to select the best candidates with ranking methods (e.g.  Pareto optimization or multi-objective optimization). These approaches are more general than a greedy approach because it does not rely on a specific ordering of the actions, but rather explores a larger space of possible sequences of actions. However, they can be computationally expensive, especially if the search space is large or the prediction model is complex. To mitigate this, we can use techniques such as pruning, parallelization, or approximation to speed up the search. 

Despite the above issues, there is still room for hope. For instance, we believe that our approach -- transforming the prediction problem and the distance metric from 3D point cloud to depth image space -- is a promising avenue of research for soft plastic object manipulation. First, it can be enriched by 2D image processing techniques (e.g., convolution, mutual information). Second, in contrast with state-of-art methods (such as Iterative Closest Point), it guarantees the fast computation and low memory usage required for robot molding -- at both learning and planning steps. Third, it could encompass a coarse mechanical model of the object. Generally, in the authors' view, the future of soft plastic object manipulation lies in the successful integration of a coarse physical model, refined with data-based techniques.
%\section{Conclusions}
%
%The contributions of our paper are the following:
%
%We define a novel scalar metric $d$ for
%measuring the distance between two point clouds. This
%metric is necessary at both layers of our framework:
%\begin{itemize}
%	\item 
%	As loss function of the autoencoders that we design and train to model the change produced by each action on a point cloud. Indeed, during training, the autoencoders of all actions should reduce the distance between their output and the ground truth.
%	\item
%	As a measure for finding the shortest path (i.e., sequence of actions) in the tree leading from initial point cloud to  final desired point cloud. 
%\end{itemize}

\section*{Acknowledgments}
The authors would like to thank Benjamin Navarro for his help on the robot programming and Andr\'e Crosnier, Fabrice Caussidery, Phlippe Fraisse, Herv\'e Louche, Serge Mora, Jean-Michel Muracciole for their suggestions on the material to be used in the experiments. 

\bibliography{biblio}
\bibliographystyle{IEEEtran}
 \vspace{-0.8cm}
 \begin{IEEEbiography}
 	[{\includegraphics[width=1in,height=1.25in,clip,keepaspectratio]{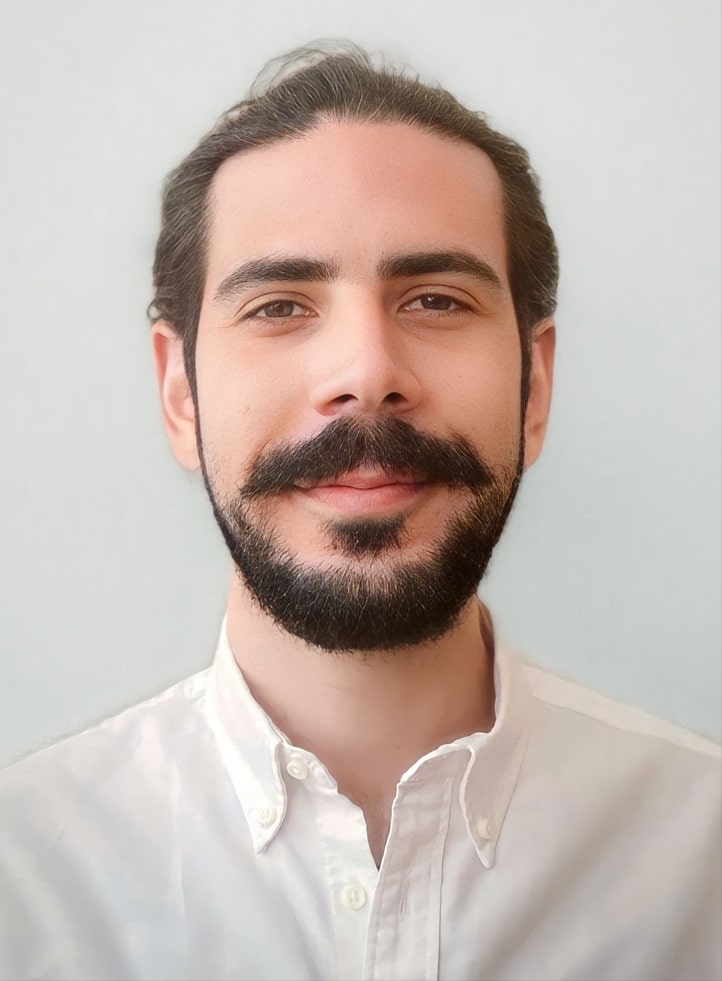}}]
 	{Ege Gursoy} received the M.Sc. in Robotics in 2020 from Universit\'e de Montpellier, France. From 2020 to 2021 he worked at UM as a research engineer in the IDH (Interactive Digital Humans) research group. Currently, he is a PhD candidate in a collaboration between Universit\'e de Montpellier and Monash University, Australia. 	
 \end{IEEEbiography}
 \vspace{-0.6cm}
 \begin{IEEEbiography}
    [{\includegraphics[width=1in,height=1.25in,clip,keepaspectratio]{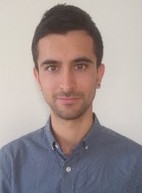}}]
 	{Sonny Tarbouriech} received the M.Sc. in Electrical and Computer Science Engineering in 2014 from the \'Ecole Polytechnique de Montpellier and a second M.Sc. in Robotics in 2016 from the University of Sherbrooke, Canada in 2016. From 2016 to 2019, he worked at Tecnalia and was a PhD student at Universit\'e de Montpellier, under the supervision of Prof. Philippe Fraisse. Since graduation, Sonny has been working at UM as a research engineer in the IDH group. 	
 \end{IEEEbiography}
 \vspace{-0.9cm}

  \begin{IEEEbiography}	[{\includegraphics[width=1in,height=1.25in,clip,keepaspectratio]{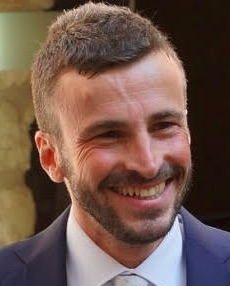}}]
 	{Andrea Cherubini} received the M.Sc. in Mechanical Engineering in 2001 from the University of Rome La Sapienza and a second M.Sc. in Control Systems in 2003 from the University of Sheffield, U.K. In 2008, he obtained the Ph.D. in Control Systems from La Sapienza, and started a three year postdoc at INRIA Rennes. Since 2011, he is at Universit\'e de Montpellier, first as Associate Professor, and then as Full Professor. At UM, he is in charge of the Robotics Master, and leads the IDH (Interactive Digital Humans) research group.
 \end{IEEEbiography}

\end{document}